\documentclass{article}
\usepackage[utf8]{inputenc}

\usepackage[preprint]{neurips_2026}

\usepackage[T1]{fontenc}    
\usepackage{hyperref}       
\usepackage{url}            
\usepackage{booktabs}       
\usepackage{amsfonts}       
\usepackage{nicefrac}       
\usepackage{microtype}      
\usepackage{xcolor}         
\usepackage{graphicx}
\usepackage{subcaption}
\usepackage{amsmath}
\usepackage{amssymb}
\usepackage{mathtools}
\usepackage{amsthm}
\usepackage{pifont}
\newcommand{\cmark}{\ding{51}}%
\newcommand{\xmark}{\ding{55}}%
\theoremstyle{plain}
\newtheorem{theorem}{Theorem}[section]

\theoremstyle{definition}

\newtheorem{assumption}[theorem]{Assumption}
\theoremstyle{remark}

\usepackage[disable]{todonotes}
\usepackage{algorithm}
\usepackage{algorithmic}
\usepackage{wrapfig}  
\title{TabQL: In-Context Q-Learning with Tabular Foundation Models}

\author{%
  Qisai Liu\\
  Department of Mechanical Engineering\\
  Iowa State University\\
  Ames, IA 50011 \\
  \And
  Zhanhong Jiang \\
  Translational AI Center\\
  Iowa State University\\
  Ames, IA 50011 \\
  \AND
  Timilehin Ayanlade \\
  Department of Mechanical Engineering\\
  Iowa State University\\
  Ames, IA 50011 \\
  \And
  Ashutosh Kumar Nirala \\
  Department of Computer Science\\
  Iowa State University\\
  Ames, IA 50011 \\
  \And
  Yang Li \\
  Independent\\
  \And
  Aditya Balu\\
  Translational AI Center\\
  Iowa State University\\
  Ames, IA 50011 \\
  \And
  Soumik Sarkar\\
  Department of Mechanical Engineering\\
  Iowa State University\\
  Ames, IA 50011 \\
}

\begin{document}

\maketitle

\begin{abstract}
We propose Tabular Q-Learning (TabQL), a reinforcement learning framework that replaces the conventional parametric Q-network in Deep Q-Learning (DQN) with a tabular foundation model endowed with in-context learning capabilities. The key idea is to represent Q-values through a sequence-to-sequence foundation model operating over a tabularized representation of state–action–Q-value tuples, enabling rapid adaptation from limited online interaction by conditioning on recent experience. TabQL departs from classical DQN by leveraging (i) zero- or few-shot Q-value inference via in-context updates, and (ii) a warm-up phase using standard DQN to bootstrap high-quality context. Particularly, to enhance the context quality, new transitions are generated by executing actions output by TabQL with predicted Q values from DQN.
We formalize TabQL, analyze its convergence and sample complexity under mild assumptions, and show that TabQL interpolates between vanilla Q-learning and DQN with in-context learning. Our analysis demonstrates that TabQL achieves improved efficiency compared to DQN by amortizing Bellman updates through in-context learning. Extensive numerical experiments with several benchmarks showcase the effectiveness and efficacy of the proposed TabQL.
\end{abstract}

\section{Introduction}\label{intro}
\vspace{-0.1in}
Reinforcement learning (RL)~\cite{kaelbling1996reinforcement,sutton1998introduction} seeks to learn optimal decision-making strategies through interaction with an environment, with Q-learning~\cite{watkins1992q} and its deep variant, Deep Q-Networks (DQN)~\cite{fan2020theoretical,mnih2013playing}, forming a central methodological backbone. By approximating the optimal action-value function using deep neural networks, DQN has enabled RL to scale to high-dimensional state spaces and achieve remarkable empirical successes~\cite{tang2025deep}, particularly in mobile robots~\cite{erkan2022mobile,yu2023hybrid,yang2020multi}, resource allocations~\cite{wu2021hybrid}, energy management systems~\cite{qin2024comparative,xiao2023ship}, manufacturing systems~\cite{he2022multi,zhang2022distributed}, and IoT services~\cite{bansal2022urbanenqosplace,liang2021toward}. Nevertheless, these gains come with well-recognized costs: DQN requires massive amounts of online interaction, is sensitive to distribution shift, and often generalizes poorly beyond the training regime.



At the same time, recent advances in foundation models~\cite{awais2025foundation,yuan2023power}, i.e., large-scale sequence models trained on broad and diverse data, have revealed a striking new capability: in-context learning (ICL)~\cite{dong2024survey,wies2023learnability}. Rather than adapting through explicit parameter updates, such models can infer task structure directly from a small number of examples provided as context and show significant adaptation/generalization capabilities~\cite{zhou2026small,yang2024context,zhang2023and}. This capability has reshaped perspectives in supervised and few-shot learning, and it naturally raises a fundamental question for RL: \textit{can action-value learning be carried out through in-context inference rather than repeated gradient-based Bellman updates?}
This work is motivated by the observation that Q-learning can be viewed not only as an optimization problem, but also as a conditional prediction problem over experience. Given a collection of past state–action–Q-value tuples, the task of estimating a Q-value for a new state–action pair can be framed as \textit{predicting the outcome of a Bellman-style computation conditioned on that experience}. A sufficiently expressive tabular foundation model (TFM), e.g. TabPFN~\cite{hollmann2022tabpfn}, TabICL~\cite{qu2025tabicl} and TabDPT~\cite{ma2024tabdpt}, can therefore operate directly on a tabularized representation of experience and perform implicit Bellman reasoning within its context window. This perspective suggests a new learning paradigm in which adaptation occurs primarily through context construction, rather than parameter modification.

\begin{wrapfigure}{r}{0.5\linewidth}
    \centering
    \includegraphics[width=\linewidth]{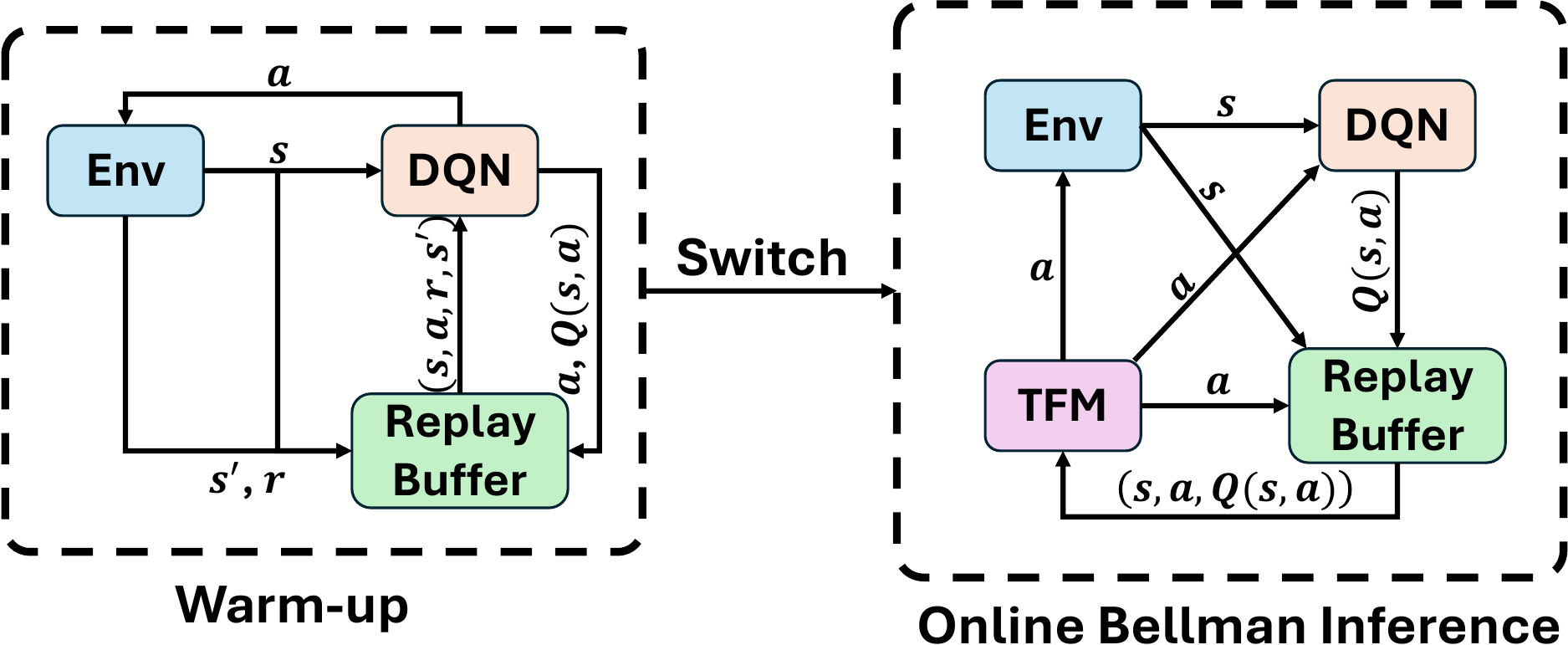}
    \caption{TabQL: A warm-up phase initializes an informative context for Q-value inference. During online Bellman inference, actions selected by the TFM via in-context learning are executed in the environment, while corresponding Q-values are predicted using the DQN trained during warm-up. The resulting transitions are incorporated into the context, enabling continual refinement of in-context Q-value inference.}
    \label{fig:tabql}
    \vspace{-0.1in}
\end{wrapfigure}

However, existing RL methods and theory are not designed around this paradigm. Classical tabular Q-learning~\cite{applebaum2022bridging} enjoys strong convergence guarantees but is infeasible in large state spaces. DQN scales effectively but relies on slow and interaction-heavy stochastic optimization and provides limited theoretical insight into generalization. Recent in-context and sequence model-based RL approaches, such as decision-transformer-style methods~\cite{chen2021decision,zheng2022online,wu2023elastic}, focus predominantly on offline or imitation-based settings and do not offer a principled integration with online Q-learning or Bellman optimality. As a result, there is a gap between the expressive power of foundation models and the algorithmic and theoretical structure of Q-learning.

\textbf{Contributions.} In this work, we bridge this gap by introducing \textit{Tabular foundational Q-Learning} (TabQL). TabQL replaces the conventional parametric Q-network in DQN with a tabular foundation model that estimates Q-values via in-context learning. To ensure stability and efficiency, TabQL incorporates a short DQN warm-up phase to populate informative experience, followed by online action selection through contextual Q-value inference. In particular, to enhance the context quality during online learning, new transitions are generated by executing actions output by TabQL, with predicted Q values from DQN, and are stored back in a replay buffer.
This design allows TabQL to interpolate between tabular Q-learning, deep Q-learning, and in-context reinforcement learning within a unified framework.

Our contributions are fourfold. 
First, we propose a new conceptual framework that reframes Q-learning as in-context inference using tabular foundation models.
Second, we introduce a practical algorithm for Foundational Q-Learning that combines DQN warm-up and contextual Q estimation. The warm-up phase aims to provide an informative context for tabular foundation models for Q-value inference. Additionally, during online learning, new transitions are generated by executing actions from TabQL with predicted Q values from DQN, therefore improving context quality.
Third, we provide a rigorous theoretical analysis establishing convergence and sample complexity guarantees, showing how in-context learning amortizes Bellman updates and reduces effective online interaction.
Finally, we validate the effectiveness and efficacy of the proposed TabQL with several benchmarks, empirically showing the enhanced sample complexity and robust generalization.

\section{Related Work}
\vspace{-0.1in}
\textbf{(Deep) Q-learning.} Classical tabular Q-learning enjoys strong theoretical guarantees, converging almost surely to the optimal Q-function under standard assumptions~\cite{watkins1992q}. However, its reliance on explicit state–action tables limits scalability. To address this, a rich literature studies approximate Q-learning with linear or nonlinear function approximation~\cite{bertsekas1996temporal}. While these methods scale better, they introduce stability and convergence challenges, and their theoretical guarantees typically require restrictive assumptions. Our work revisits tabular representations, not as explicit tables, but as inputs to expressive sequence models that can generalize across tasks.
DQN combines Q-learning with deep neural networks to approximate action-value functions in large state spaces. The original DQN algorithm introduced experience replay and target networks to stabilize training and demonstrated human-level performance on Atari benchmarks~\cite{mnih2013playing}. Subsequent extensions addressed specific failure modes, including overestimation bias (Double DQN)~\cite{van2016deep}, value-function decomposition (Dueling DQN)~\cite{wang2016dueling}, and distributional value estimation~\cite{bellemare2017distributional}. Despite these advances, DQN-style methods remain highly sample-inefficient and rely on repeated gradient-based Bellman updates, motivating alternative paradigms for more efficient learning. 

\textbf{Transformers with in-context learning.} In-context learning (ICL) refers to the ability of large sequence models to adapt to new tasks by conditioning on a small number of examples without parameter updates. This phenomenon was popularized by large language models~\cite{brown2020language} and has since been studied theoretically as implicit Bayesian inference or meta-learning~\cite{garg2022can,von2023transformers}. Recent work has shown that Transformers trained on synthetic tasks can implement learning algorithms in-context, including linear regression and bandit strategies~\cite{li2023transformers}. These findings suggest that foundation models can amortize learning procedures through context, providing the conceptual basis for TabQL.
Sequence-model-based approaches have recently been applied to RL by framing trajectories as sequences and using autoregressive prediction for control. Notable examples include Decision Transformers~\cite{chen2021decision} and trajectory Transformers~\cite{janner2021offline}. While these methods demonstrate strong performance in offline and imitation-based settings~\cite{xu2023hyper}, they do not explicitly model Bellman optimality or Q-functions. Moreover, they typically do not address online exploration and convergence guarantees yet. In contrast, our approach preserves the structure of Q-learning while leveraging sequence models for in-context value estimation. Please see Appendix~\ref{additional_related_work} for more related works.

\section{Problem Formulation and Preliminaries}\label{sec:problem_formulation}
\vspace{-0.1in}

In this section, we formalize the reinforcement learning setting, introduce all necessary notation, and define the objects used throughout the paper. We place particular emphasis on making explicit the connection between classical Q-learning and its realization via tabular foundation models.

\textbf{Markov Decision Process.}
We consider a discounted Markov Decision Process (MDP) $\mathcal{M}=(\mathcal{S},\mathcal{A},P, r,\gamma)$, where $\mathcal{S}$ is a state space, $\mathcal{A}$ is a finite action space, $P(\cdot|s,a)$ is the transition kernel, $r:\mathcal{S}\times \mathcal{A}\to[0,1]$ is the reward function, and $\gamma\in(0,1)$ is the discount factor. A (possibly stochastic) policy $\pi:\mathcal{S}\to\Delta(\mathcal{A})$ induces a distribution over trajectories $(s_0,a_0,r_0,s_1,...)$, with $\Delta(\cdot)$ being a probability function. We define the action-value function of a policy $\pi$ as
$
    Q^\pi(s,a):=\mathbb{E}_\pi[\sum_{t=0}^\infty\gamma^tr(s_t,a_t)|s_0=s,a_0=a]
$ and the optimal action-value function
$
    Q^*(s,a):=\text{sup}_\pi Q^\pi(s,a).
$

\textbf{Notation for Q-value estimates.} Throughout the paper, we distinguish three Q-value quantities: (i) $Q^*$ denotes the true optimal action-value function defined above; (ii) $\hat{Q}^{\mathrm{DQN}}_t$ denotes the action-value estimate produced by the warm-up DQN at step $t$; (iii) $\hat{Q}^{\mathrm{TFM}}_t$ denotes the estimate produced by the tabular foundation model via in-context inference, $\hat{Q}^{\mathrm{TFM}}_t(s,a) := f_\phi(\mathcal{C}_t, s, a)$.
Define the Bellman optimality operator $\mathcal{T}:\mathbb{R}^{|\mathcal{S}|\times |\mathcal{A}|}\to\mathbb{R}^{|\mathcal{S}|\times |\mathcal{A}|}$
\begin{equation}
    (\mathcal{T}Q)(s,a):=r(s,a)+\gamma\mathbb{E}_{s'\sim P(\cdot|s,a)}\bigg[\text{max}_{a'\in\mathcal{A}}Q(s',a')\bigg].
\end{equation}
It is well-known that $\mathcal{T}$ is a $\gamma$-contraction~\cite{kadurha2025bellman} under the sup norm $\|\cdot\|_\infty$, and $Q^*$ is the unique fixed point of $\mathcal{T}$. In tabular Q-learning, the action-value estimates are updated via stochastic approximation:
\begin{equation}\label{eq_4}
    Q_{t+1}(s_t,a_t)=(1-\alpha_t)Q_t(s_t,a_t)+\alpha_t(r_t+\gamma\text{max}_{a'}Q_t(s_{t+1},a')),
\end{equation}
where $\{\alpha_t\}$ is a learning-rate sequence. Deep Q-learning (DQN) replaces the table with a parametric function $Q_\theta:\mathcal{S}\times \mathcal{A}\to\mathbb R$, trained by minimizing the empirical Bellman error over a replay buffer using stochastic gradient descent. To leverage tabular foundation model for Q value inference, we define the \textit{tabularization of experience}. We represent each transition as a tuple $x_i=(s_i,a_i,r_i,s'_i,\hat{Q}^{\mathrm{DQN}}(s_i,a_i))\in\mathcal{X}$, where the final coordinate is the warm-up DQN's Q-estimate used as a soft label, $\mathcal{X}:=\mathcal{S}\times\mathcal{A}\times\mathbb R\times \mathcal{S}\times\mathbb R$. This is the augmentation of the transition in vanilla DQN, which is represented by $(s,a,r,s')$. Q-value inference is the primary focus when leveraging the TFM. Each tuple supplies features $(s,a)$ and the corresponding label $\hat{Q}^{\mathrm{DQN}}(s,a)$ as contextual information, while $r$ and $s'$ are retained for defining the empirical Bellman operator used in our analysis. A dataset or replay buffer is denoted by $\mathcal{D}=\{x_i\}_i^N$. At time $t$, we define a context window of fixed length $K$ as $\mathcal{C}_t:=(x_{t-K}, x_{t-K+1},...,x_{t-1})\in\mathcal{X}^K$. This context constitutes the sole mechanism by which the model adapts during online interaction. 

\textbf{Tabular Foundation Model.} A TFM is a frozen pretrained sequence model $f_\phi: \mathcal{X}^K \times (\mathcal{S}\times\mathcal{A}) \to \mathbb{R}$ that maps a context $\mathcal{C}$ and query $(s,a)$ to a scalar prediction. The TFM's Q-estimate at step $t$ is $\hat{Q}^{\mathrm{TFM}}_t(s,a) := f_\phi(\mathcal{C}_t, s, a)$, and TabQL acts greedily as $\pi_{\phi,\mathcal{C}_t}(s) \in \arg\max_a \hat{Q}^{\mathrm{TFM}}_t(s,a)$. The notation $\hat{Q}^{\mathrm{TFM}}_t$ denotes the TFM's Q-estimate at step $t$, computed via in-context inference on the context $\mathcal{C}_t$. Pretraining and adaptation properties of the TFM are detailed in Appendix~\ref{app:tfm_full}.




\section{TabQL}\label{sec:tabql}
\vspace{-0.1in}
In this section, we introduce Tabular foundational Q-Learning (TabQL), a reinforcement learning framework that replaces the parametric Q-network in DQN with a Tabular Foundation Model (TFM) capable of performing in-context Bellman inference. We first establish the conceptual foundations underlying TabQL, then describe the algorithm in full detail, followed by technical discussion. 

\textbf{Memory decomposition.} TabQL maintains two complementary representations of accumulated experience. The warm-up DQN encodes \emph{global} value information across the entire state-action space in its parameters $\theta$, providing a coarse but persistent prior that does not depend on the current trajectory. The rolling context $\mathcal{C}_t$ encodes \emph{local, recent} value information about the region of the state-action space currently being visited, sufficient for the TFM to perform fine-grained in-context inference. The DQN's role is to supply approximate Q-labels and to maintain coverage over $(s,a)$ pairs that are not present in the recent context; the TFM's role is to consume those labels and produce refined Q-estimates conditioned on local trajectory structure. This decomposition explains both TabQL's strengths and its limits: when context revisitation is high, the TFM dominates and improves on the DQN; when context revisitation is low (e.g., very large state spaces with rare returns to the same $(s,a)$), the TFM falls back to interpolating from the DQN's labels and TabQL degrades gracefully toward DQN's performance rather than failing catastrophically.

\textbf{In-context Bellman inference.} TabQL treats Q-learning as a conditional inference problem: given a context $\mathcal{C}_t = \{(s_i, a_i, r_i, s'_i, \hat{Q}^{\mathrm{DQN}}(s_i, a_i))\}_{i=1}^K \subset \mathcal{B}$, infer the Q-value of a query pair $(s,a)$. We define the empirical Bellman operator induced by $\mathcal{C}_t$:
\begin{equation}\label{eq:empirical_bellman}
(\hat{\mathcal{T}}_{\mathcal{C}_t}\hat{Q})(s,a):=\mathbb{E}_{s'\sim \hat{P}_{\mathcal{C}_t}(\cdot|s,a)}\bigg[r(s,a)+\gamma f_\phi(\mathcal{C}_t,s',a^\star_{\hat{Q}}(s'))\bigg],
\end{equation}
where $a^\star_{\hat{Q}}(s'):=\arg\max_{a'\in\mathcal{A}}\hat{Q}(s',a')$ and $\hat{P}_{\mathcal{C}_t}(\cdot|s,a)$ is the empirical next-state distribution from transitions in $\mathcal{C}_t$ that share $(s,a)$ (or, when no such transitions exist, the nearest-neighbor set under the TFM's input metric — see Assumption~\ref{assumption_3}). The TFM $f_\phi$ approximates the fixed point of $\hat{\mathcal{T}}_{\mathcal{C}_t}$, amortizing Bellman updates through in-context learning rather than gradient descent. Crucially, no task-specific pretraining is required: TabPFN and TabDPT are pretrained on synthetic tabular regression and apply their general-purpose smoothness prior to the noisy DQN labels in $\mathcal{C}_t$, yielding lower-variance Q-estimates than the raw labels themselves. We elaborate on the conceptual setup, the connection to classical Q-learning, and the intuition for why this works in Appendix~\ref{app:icl_bellman_full}.

\begin{algorithm}[tb]
  \caption{Tabular foundational Q-Learning (TabQL)}
  \label{alg:fql}
  \begin{algorithmic}[1]
    \STATE {\bfseries Input:} discount factor $\gamma$, warm-up length $T_0$, buffer size $W$, context size $K$, exploration schedule $\varepsilon_t$, pretrained TFM $f_\phi$
    \STATE Initialize replay buffer $\mathcal{B} \leftarrow \emptyset$, DQN parameters $\theta$
    \STATE \texttt{\#\#\# Warm-up Phase \#\#\#}
    \FOR{$t = 1, 2, \ldots, T_0$}
      \STATE Select $a_t$ via $\varepsilon_t$-greedy on $\hat{Q}^{\mathrm{DQN}}_t(s_t, \cdot)$, execute, observe $(r_t, s_{t+1})$
      \STATE Append $(s_t, a_t, r_t, s_{t+1}, \hat{Q}^{\mathrm{DQN}}_t(s_t, a_t))$ to $\mathcal{B}$ (evicting oldest if $|\mathcal{B}| > W$)
      \STATE Update DQN parameters $\theta$ via standard TD-loss minimization
    \ENDFOR
    \STATE \texttt{\#\#\# In-Context Bellman Inference Phase \#\#\#}
    \FOR{$t = T_0 + 1, T_0 + 2, \ldots$}
      \STATE Sample context $\mathcal{C}_t \subseteq \mathcal{B}$ of size $K$ (most recent transitions)
      \STATE Compute $\hat{Q}^{\mathrm{TFM}}_t(s_t, a) = f_\phi(\mathcal{C}_t, s_t, a)$ for each $a \in \mathcal{A}$
      \STATE Select $a_t$ via $\varepsilon_t$-greedy on $\hat{Q}^{\mathrm{TFM}}_t(s_t, \cdot)$, execute, observe $(r_t, s_{t+1})$
      \STATE Append $(s_t, a_t, r_t, s_{t+1}, \hat{Q}^{\mathrm{DQN}}_t(s_t, a_t))$ to $\mathcal{B}$ (evicting oldest)
      \STATE Update DQN parameters $\theta$ via TD-loss on a mini-batch from $\mathcal{B}$
      \STATE $s_t \leftarrow s_{t+1}$
    \ENDFOR
    \STATE \texttt{\#\#\# Implementation note \#\#\#} A quality-check gate is applied to the context $\mathcal{C}_t$ in both phases (before and after switching), filtering out transitions whose labels fail the quality criterion before they are passed to the TFM.
  \end{algorithmic}
\end{algorithm}

\textbf{Algorithmic framework.}
Algorithm~\ref{alg:fql} depicts the working mechanism of TabQL. We now analyze the algorithm in detail. The warm-up phase (Lines 2–4) ensures that early contexts are informative by relying on a standard DQN policy. 
We denote by $\hat{Q}(s,a)$ the Q-value predicted by the pretrained DQN. Additionally, in practical implementation, we augment the transition with the time $t$, which provides more information through the time feature.
The choice of when to switch from DQN to the tabular foundation model is currently treated as a tunable hyperparameter. This design is intentional rather than ad hoc: the warm-up phase serves only to ensure that the context contains non-degenerate estimates of value and transition structure, after which in-context Bellman inference becomes effective. In practice, the required warm-up horizon is modest, as TabQL does not rely on precise Q-value convergence but only on coarse signal to bootstrap contextual reasoning. Developing adaptive or statistically grounded switching criteria, e.g., based on Bellman residuals or context coverage, is a promising direction for future work and does not affect the validity of our convergence or sample complexity results. 
In context construction (Line 6), we sample from the most recent transitions instead of uniformly sampling from the replay buffer, as recent samples have been empirically observed to yield superior performance. The context size $|\mathcal{C}_t|$ may also be varied over time to improve the adaptivity of in-context learning. We conjecture that smaller contexts are sufficient as $t$ increases, since the model progressively consolidates knowledge that generalizes well. While a principled treatment of adaptive context sizing is of independent interest, we defer it to future work and instead empirically study the effect of different context sizes in our ablation experiments.

From a technical standpoint, TabQL replaces incremental stochastic approximation with amortized Bellman inference (Line 8). Each new transition immediately affects Q-value predictions through inclusion in $\mathcal{C}$, without waiting for gradient descent. This reduces update variance, since inference is deterministic given context (Appendix~\ref{stochasticity}), avoiding noisy SGD updates as evidenced empirically. The quality gate on $\mathcal{C}_t$ is inspired by~\cite{schiff2025gradient}.

Intuitively, the TFM acts as a \textit{learned Q-learning} algorithm. The context $\mathcal{C}$ plays the role of a temporary Q-table, and the model’s forward pass performs the computation that classical Q-learning would require numerous updates to achieve. During environment interaction (Line 10), actions are selected using the TFM rather than the DQN policy. To maintain informative contexts during online interaction, we continue to use DQN-predicted Q-values as supervision. (see Appendix~\ref{dependence}) 

\section{Theoretical Analysis}\label{sec:theory}
\vspace{-0.1in}

In this section, we state the theoretical analysis for TabQL, with proofs deferred to Appendix~\ref{missing_proof}. We recall the corrected empirical Bellman operator from Section~\ref{sec:tabql} (Eq.~\ref{eq:empirical_bellman}):
\begin{equation*}
(\hat{\mathcal{T}}_{\mathcal{C}_t}\hat{Q})(s,a):=\mathbb{E}_{s'\sim \hat{P}_{\mathcal{C}_t}(\cdot|s,a)}\bigg[r(s,a)+\gamma f_\phi(\mathcal{C}_t,s',a^\star_{\hat{Q}}(s'))\bigg].
\end{equation*}
For a fixed context $\mathcal{C}_t$, this operator is a deterministic $\gamma$-contraction. The TFM produces the next Q-iterate via in-context inference:
$
    \hat{Q}^{\mathrm{TFM}}_{t+1}(s,a) = f_\phi(\mathcal{C}_t, s, a).
$
We assume the following ICL approximation condition, which replaces the role of SGD noise in DQN. 

\begin{assumption}\label{assumption_1}
There exists $\varepsilon_{\mathrm{ICL}} > 0$ such that for all $t$,
$\sup_{s,a}\big|\hat{Q}^{\mathrm{TFM}}_{t+1}(s,a) - (\hat{\mathcal{T}}_{\mathcal{C}_t}\hat{Q}^{\mathrm{TFM}}_t)(s,a)\big| \leq \varepsilon_{\mathrm{ICL}}$.
\end{assumption}
This assumption captures model capacity limits, finite context length, and pretraining mismatch. Crucially, this error does not accumulate through SGD. We now decompose the total error between $Q_{t+1}$ and $Q^*$ into multiple sources:


\begin{equation}\label{eq_15}
\hat{Q}^{\mathrm{TFM}}_{t+1}-Q^* = \underbrace{(\mathcal{T}\hat{Q}^{\mathrm{TFM}}_t-\mathcal{T}Q^*)}_{\text{Bellman contraction}} + \underbrace{(\hat{\mathcal{T}}_{\mathcal{C}_t}\hat{Q}^{\mathrm{TFM}}_t-\mathcal{T}\hat{Q}^{\mathrm{TFM}}_t)}_{\text{statistical error (incl.}\,\varepsilon_{\mathrm{label}})}+ \underbrace{(\hat{Q}^{\mathrm{TFM}}_{t+1}-\hat{\mathcal{T}}_{\mathcal{C}_t}\hat{Q}^{\mathrm{TFM}}_t)}_{\text{ICL error}}.
\end{equation}

Before bounding each term in the above equation, we justify why Assumption~\ref{assumption_1} is a reasonable assumption. It mirrors the standard approximation assumptions used in fitted Q-iteration~\cite{antos2007fitted} and approximate dynamic programming~\cite{busoniu2017reinforcement}. In TabQL, approximation error decomposes into (i) model error (i.e., ICL error as in Eq.~\ref{eq_15}) from amortized Bellman inference and (ii) statistical error from finite context. As context size grows, the latter vanishes at a certain rate (shown in label approximation error below),
while the former depends on pretraining and model capacity. This assumption is therefore no stronger than classical realizability and is strictly weaker than assuming direct representation of $Q^*$.
We next bound each term separately. For the Bellman term, we have by the standard properties of $\mathcal{T}$, 
$
    \|\mathcal{T}\hat{Q}^{\mathrm{TFM}}_t-\mathcal{T}Q^*\|_\infty\leq \gamma\|\hat{Q}^{\mathrm{TFM}}_t-Q^*\|_\infty.
$
We then bound the contextual approximation error. First, we let
$\varepsilon_{\mathrm{stat}}(t) := \sup_{s,a}\big|(\hat{\mathcal{T}}_{\mathcal{C}_t}\hat{Q}^{\mathrm{TFM}}_t)(s,a) - (\mathcal{T}\hat{Q}^{\mathrm{TFM}}_t)(s,a)\big|$.
To this end, we introduce an assumption regarding the generalization of TFM as follows.

\begin{assumption}\label{assumption_2}
The TFM $f_\phi\in\mathcal{F}$ satisfies $\sup_{s,a}|f_\phi(\mathcal{C},s,a)-Q^*(s,a)|\leq\epsilon$ with high probability over a context of size $K$, for a function class $\mathcal{F}$ with covering number $\mathcal{N}(\mathcal{F},\epsilon)$. A more technical justification appears in Appendix~\ref{justification}.
\end{assumption}

\begin{assumption}\label{assumption_3}
For each query pair $(s,a)$ encountered during execution, the context $\mathcal{C}_t$ is sufficiently rich for the TFM to estimate $\hat{P}_{\mathcal{C}_t}(\cdot|s,a)$ at the resolution required by Assumption~\ref{assumption_1}. In finite tabular settings this corresponds to at least $m_{\min}$ transitions in $\mathcal{C}_t$ sharing the query state $s$; in higher-dimensional settings it corresponds to sufficient nearest-neighbor coverage under the TFM's input metric.
\end{assumption}
Assumption~\ref{assumption_2} permits standard concentration arguments (Hoeffding plus union bound)~\cite{hoeffding1963probability,hertz2020improved}, while Assumption~\ref{assumption_3} guarantees the per-query sample size $m_t$ needed for those bounds. The latter is non-trivial: it holds in finite tabular MDPs with sufficient $\varepsilon$-greedy exploration and a buffer larger than $|\mathcal{S}|\cdot m_{\min}$, but fails in high-dimensional state spaces where exact $(s,a)$ revisitation within a finite context window approaches zero probability. In that regime, TabQL falls back to nearest-neighbor inference under the TFM's input metric and degrades gracefully toward the warm-up DQN's performance; we treat this as an explicit limitation in Section~\ref{sec:experiments}.

\textbf{Label approximation error.} A subtle but important consideration is that transitions in the buffer carry the label $\hat{Q}^{\mathrm{DQN}}(s,a)$ rather than the unobserved true value $Q^*(s,a)$. This induces a residual error
$
\varepsilon_{\mathrm{label}} := \|\hat{Q}^{\mathrm{DQN}} - Q^*\|_\infty,
$
inherited from the warm-up DQN at the moment of switching. This error is fixed once the warm-up phase ends and does not iterate or compound during ICL inference. It enters our analysis through the empirical kernel $\hat{P}_{\mathcal{C}_t}$, contributing additively to the statistical error term. Crucially, $\varepsilon_{\mathrm{label}} \to 0$ as the warm-up DQN converges to $Q^*$, formally justifying the warm-up design: a sufficiently long warm-up minimizes label error, while a premature switch leaves $\varepsilon_{\mathrm{label}}$ large and dominates the asymptotic gap.
Combining Hoeffding concentration on the empirical kernel with the label-induced bias gives, with probability at least $1-\delta$ ($\delta>0$):
$
\varepsilon_{\mathrm{stat}}(t) \;\leq\; \gamma\|\hat{Q}^{\mathrm{TFM}}_t\|_\infty\sqrt{\frac{2\log(|\mathcal{S}||\mathcal{A}|/\delta)}{m_t}} + \varepsilon_{\mathrm{label}}.
$
In TabQL, per-step errors are bounded by a fixed constant independent of $t$, in contrast to DQN where per-step errors are themselves functions of accumulated gradient noise. The standard $1/(1-\gamma)$ horizon factor still appears in our bounds — TabQL's advantage is the absence of \emph{gradient-propagated} compounding, not the absence of the discount factor. We are now ready to state the main result.

\begin{theorem}\label{theorem_1}
Let Assumptions~\ref{assumption_1}--\ref{assumption_3} hold. With probability at least $1-\delta$, for all $t \geq 1$,
$
\|\hat{Q}^{\mathrm{TFM}}_t - Q^*\|_\infty \leq \gamma^t\|\hat{Q}^{\mathrm{TFM}}_0 - Q^*\|_\infty + \sum_{\tau=0}^{t-1}\gamma^\tau\big(\varepsilon_{\mathrm{ICL}} + \varepsilon_{\mathrm{stat}}(\tau)\big),
$
where $\varepsilon_{\mathrm{stat}}(\tau) \leq \frac{\gamma}{1-\gamma}\sqrt{\frac{2\log(|\mathcal{S}||\mathcal{A}|/\delta)}{m_\tau}} + \varepsilon_{\mathrm{label}}$, with $m_\tau$ the per-query revisitation count guaranteed by Assumption~\ref{assumption_3}.
\end{theorem}

\textbf{Asymptotic behavior and performance loss.} Taking $t \to \infty$ in Theorem~\ref{theorem_1} and applying the dominated convergence theorem (full derivation in Appendix~\ref{asymptotic_derivation}), the limsup of the value error is bounded by $\mathcal{O}((\varepsilon_{\mathrm{ICL}} + \varepsilon_{\mathrm{label}})/(1-\gamma))$ plus a residual statistical term that vanishes as $m_{\min}$ grows. Combining with the standard performance-difference inequality~\cite{antos2008learning,munos2008finite} yields the asymptotic suboptimality:
$
V^*(s) - V^{\pi_\infty}(s) \;\leq\; \frac{2\gamma(\varepsilon_{\mathrm{ICL}} + \varepsilon_{\mathrm{label}})}{(1-\gamma)^2} + \mathcal{O}\!\left(\frac{1}{(1-\gamma)^3}\sqrt{\frac{\log(|\mathcal{S}||\mathcal{A}|/\delta)}{m_{\min}}}\right).
$
TabQL therefore achieves near-optimal control whenever both $\varepsilon_{\mathrm{ICL}}$ and $\varepsilon_{\mathrm{label}}$ are small. The error structure differs fundamentally from DQN: DQN's per-step error is driven by SGD gradient noise that compounds across iterations, while TabQL's per-step error is bounded by fixed constants that do not grow with $t$. Both methods inherit the $1/(1-\gamma)$ horizon factor, but TabQL avoids the $\sum_t \alpha_t^2$ variance accumulation; we elaborate this comparison in Appendix~\ref{dqn_comparison_full}.

\begin{theorem}\label{theorem_2}
Let Assumptions~\ref{assumption_1}--\ref{assumption_3} hold, rewards be bounded in $[0,1]$, and suppose both $\varepsilon_{\mathrm{ICL}} \leq c(1-\gamma)\epsilon$ and $\varepsilon_{\mathrm{label}} \leq c(1-\gamma)\epsilon$ for some constant $c < 1/3$. Then TabQL achieves $\|\hat{Q}^{\mathrm{TFM}}_t - Q^*\|_\infty \leq \epsilon$ with probability at least $1-\delta$ after a total of
$
N=\tilde{\mathcal{O}}\!\left(\frac{1}{(1-\gamma)^4\epsilon^2}\log\frac{|\mathcal{S}||\mathcal{A}|}{\delta}+\frac{1}{1-\gamma}\log\frac{1}{\epsilon}\right).
$
environment interactions, where the first term is the context-quality cost and the second is the contraction horizon.
\end{theorem}

Theorem~\ref{theorem_2} states a conditional sample complexity: the bound holds when the warm-up phase produces a sufficiently accurate teacher (small $\varepsilon_{\mathrm{label}}$) and the TFM achieves uniform regression accuracy (small $\varepsilon_{\mathrm{ICL}}$). The complexity covers both the warm-up interactions and the subsequent online interactions used to build the context, since each environment step adds one transition to the buffer. When the assumptions are satisfied, the dominant term is $\tilde{\mathcal{O}}(1/((1-\gamma)^4\epsilon^2))$.

\textbf{Conditional comparison with DQN.} The standard DQN sample complexity is $N_{\mathrm{DQN}}=\tilde{\mathcal{O}}(|\mathcal{S}||\mathcal{A}|/((1-\gamma)^4\epsilon^2))$~\cite{li2024q}, dominated by the per-state-action visitation cost. Under Assumptions~\ref{assumption_1}--\ref{assumption_3}, TabQL's interaction cost scales as $\tilde{\mathcal{O}}(c \cdot m_{\min} / \varepsilon_{\mathrm{ICL}}^2)$, where $c$ is the visitation constant in Theorem~\ref{theorem_2} and $m_{\min}$ is the per-query revisitation count required by the TFM (full derivation in Appendix~\ref{dqn_comparison_full}). The product $c \cdot m_{\min}$ replaces DQN's $|\mathcal{S}||\mathcal{A}|$ factor. Under uniform exploration, $c = \Theta(|\mathcal{S}||\mathcal{A}|)$ and the apparent gain vanishes. However, when exploration concentrates on the relevant subset of the state-action space and the TFM achieves uniform accuracy from a sub-state-space-sized context (so that $c \cdot m_{\min} \ll |\mathcal{S}||\mathcal{A}|$), TabQL achieves a strict improvement.

\textbf{Early switching failure.} A premature switch (small $T_0$) leaves both the initial approximation error $\|\hat{Q}^{\mathrm{TFM}}_0 - Q^*\|_\infty$ and $\varepsilon_{\mathrm{label}}$ large, and the asymptotic bound becomes dominated by $\varepsilon_{\mathrm{label}}/(1-\gamma)$ — a quantity independent of $t$ and $m_{\min}$ that cannot be corrected by additional context or longer ICL inference, since biased labels drive the contraction toward a biased fixed point rather than $Q^*$. The corrected theory therefore predicts a threshold behavior: $T_0$ must be large enough for $\varepsilon_{\mathrm{label}}$ to fall below the level at which the contraction can recover, after which further increases yield no additional benefit. This matches the empirical pattern we observe in Section~\ref{sec:experiments}; we provide the full mechanistic analysis in Appendix~\ref{early_switching_full}.

\section{Numerical Results}\label{sec:experiments}
\vspace{-0.1in}

We evaluate TabQL across three Gymnasium grid-world environments — Taxi-v3, CliffWalking-v1, and FrozenLake-v1~\cite{beikmohammadi2022nars} — and study three aspects: (i) sample efficiency relative to baselines, (ii) generalization across initial conditions, and (iii) sensitivity to key hyperparameters. All experiments use five random seeds; reported curves show mean returns with shaded standard deviation, and baseline hyperparameters follow standard settings from Hugging Face and RL Zoo~\cite{sanghi2024foundation}.



\textbf{Comparative study.} We compare TabQL against five baselines: tabular Q-learning, DQN, Double DQN, Dueling DQN, and Fitted Q-Iteration (FQI)~\cite{antos2007fitted}. Figures~\ref{fig:Cliffwalking_tabQL}--\ref{fig:FrozenLake_tabQL} show that both TabQL variants (TabPFN and TabDPT backbones) reach near-optimal performance with substantially fewer environment interactions than the deep RL baselines. The qualitative pattern is consistent across all three environments: TabQL exhibits a sharp transition from baseline-DQN behavior to near-optimal returns immediately after the ICL switching point ($T_0$), often within a few additional steps. The deep RL baselines (DQN, Double DQN, Dueling DQN) require an order of magnitude more interactions to reach comparable returns. Tabular Q-learning eventually converges on Taxi and CliffWalking, but is much slower than TabQL. FQI achieves competitive asymptotic performance on Taxi and CliffWalking, comparable to TabQL, but operates as a batch algorithm rather than online. It requires the full dataset upfront and produces no learning curve. We include it as a strong asymptotic baseline. On CartPole-v1 (Figure~\ref{fig:CartPole_tabQL}), TabQL with the TabPFN backbone reaches the optimal return of 500 substantially faster than DQN, Double DQN, and Dueling DQN, demonstrating that the framework extends beyond discrete observation spaces. The two TabQL variants (TabPFN and TabDPT) perform similarly across environments, suggesting the result is robust to the specific choice of TFM backbone.



\begin{figure*}[t]
    \centering
    \begin{subfigure}[t]{0.48\textwidth}
        \centering
        \includegraphics[width=\linewidth]{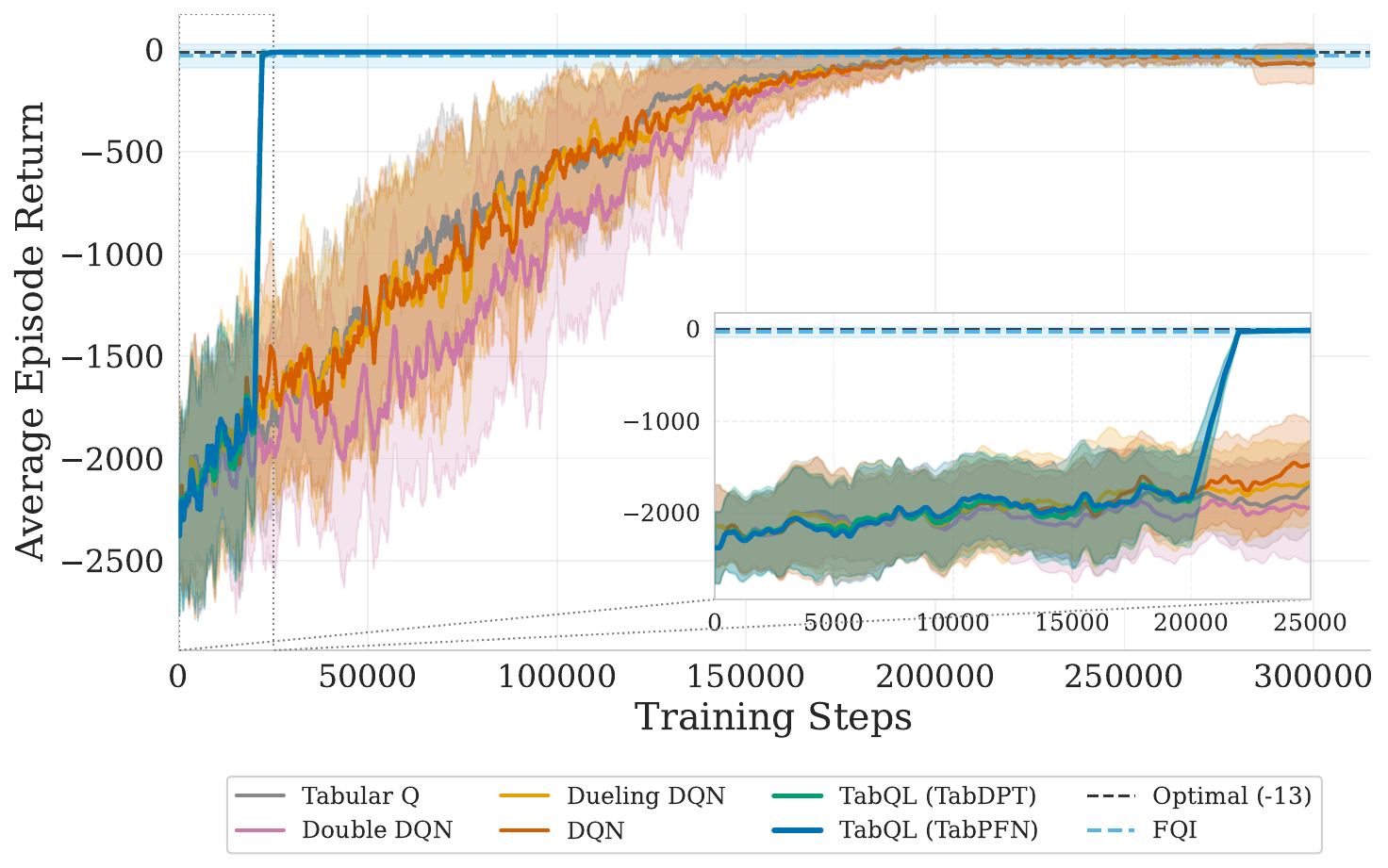}
        \caption{CliffWalking-v1}
        \label{fig:Cliffwalking_tabQL}
    \end{subfigure}
    \hfill
    \begin{subfigure}[t]{0.48\textwidth}
        \centering
        \includegraphics[width=\linewidth]{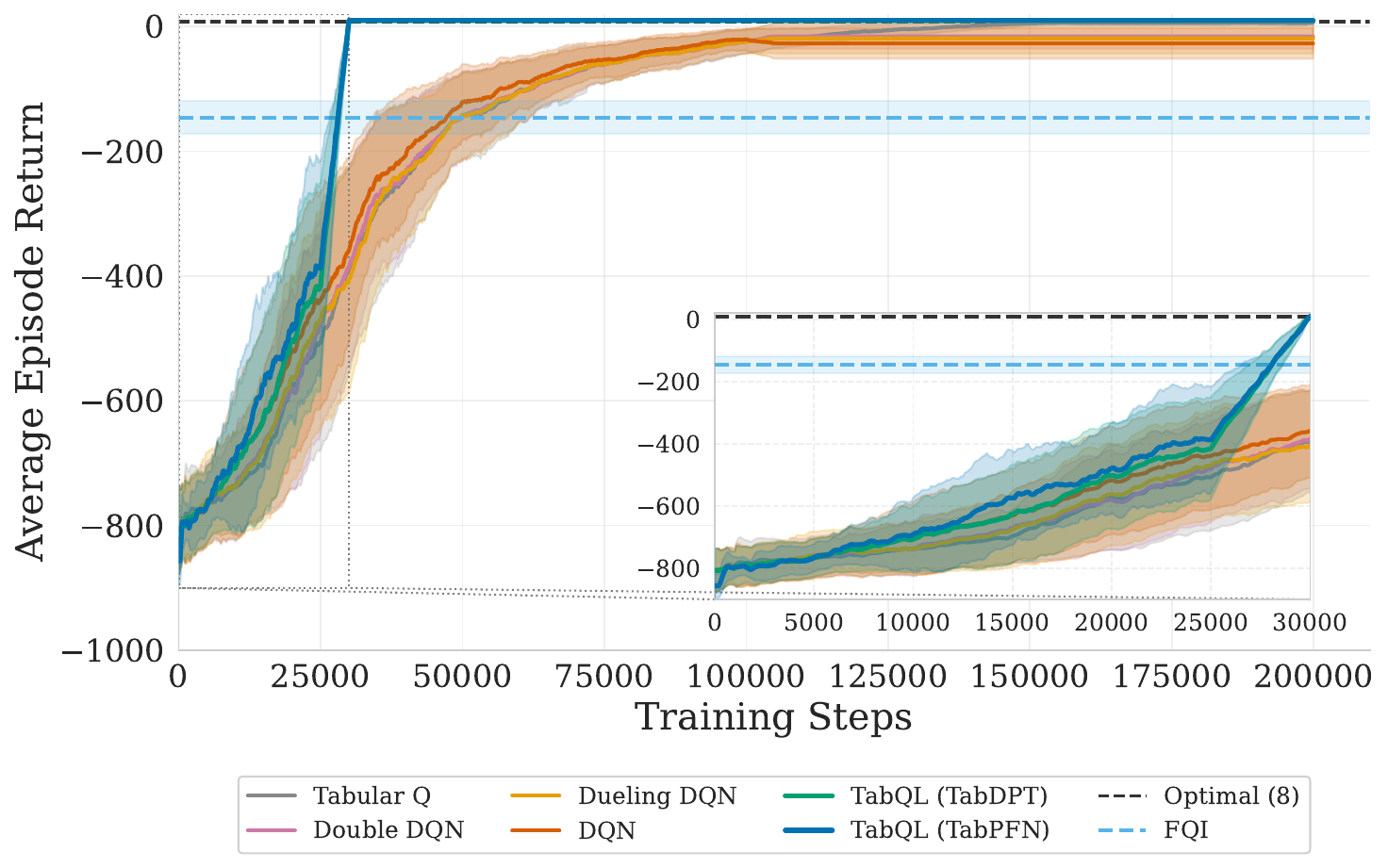}
        \caption{Taxi-v3}
        \label{fig:Taxi_tabQL}
    \end{subfigure}


    \begin{subfigure}[t]{0.48\textwidth}
        \centering
        \includegraphics[width=\linewidth]{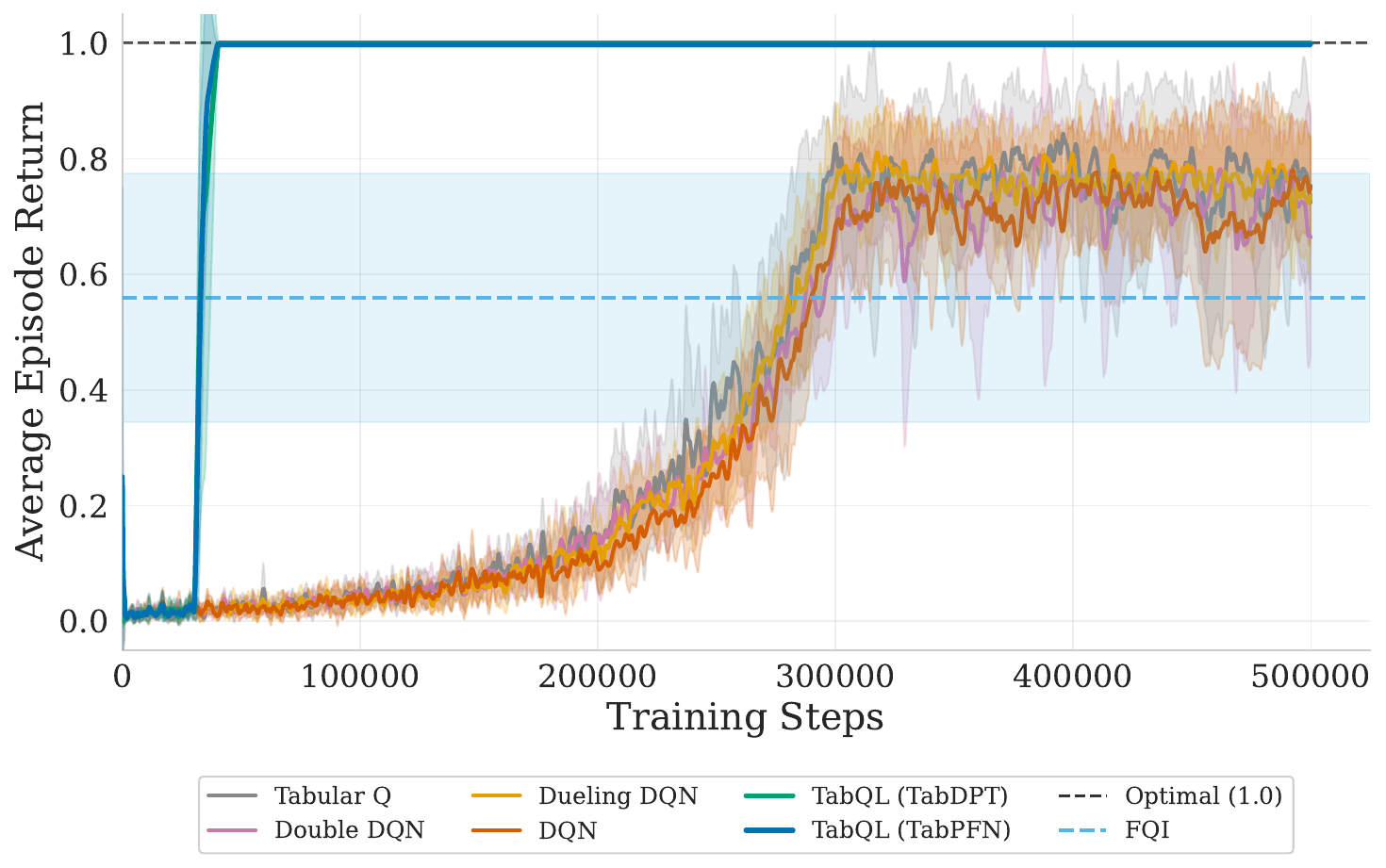}
        \caption{FrozenLake-v1}
        \label{fig:FrozenLake_tabQL}
    \end{subfigure}
    \hfill
    \begin{subfigure}[t]{0.48\textwidth}
        \centering
        \includegraphics[width=\linewidth]{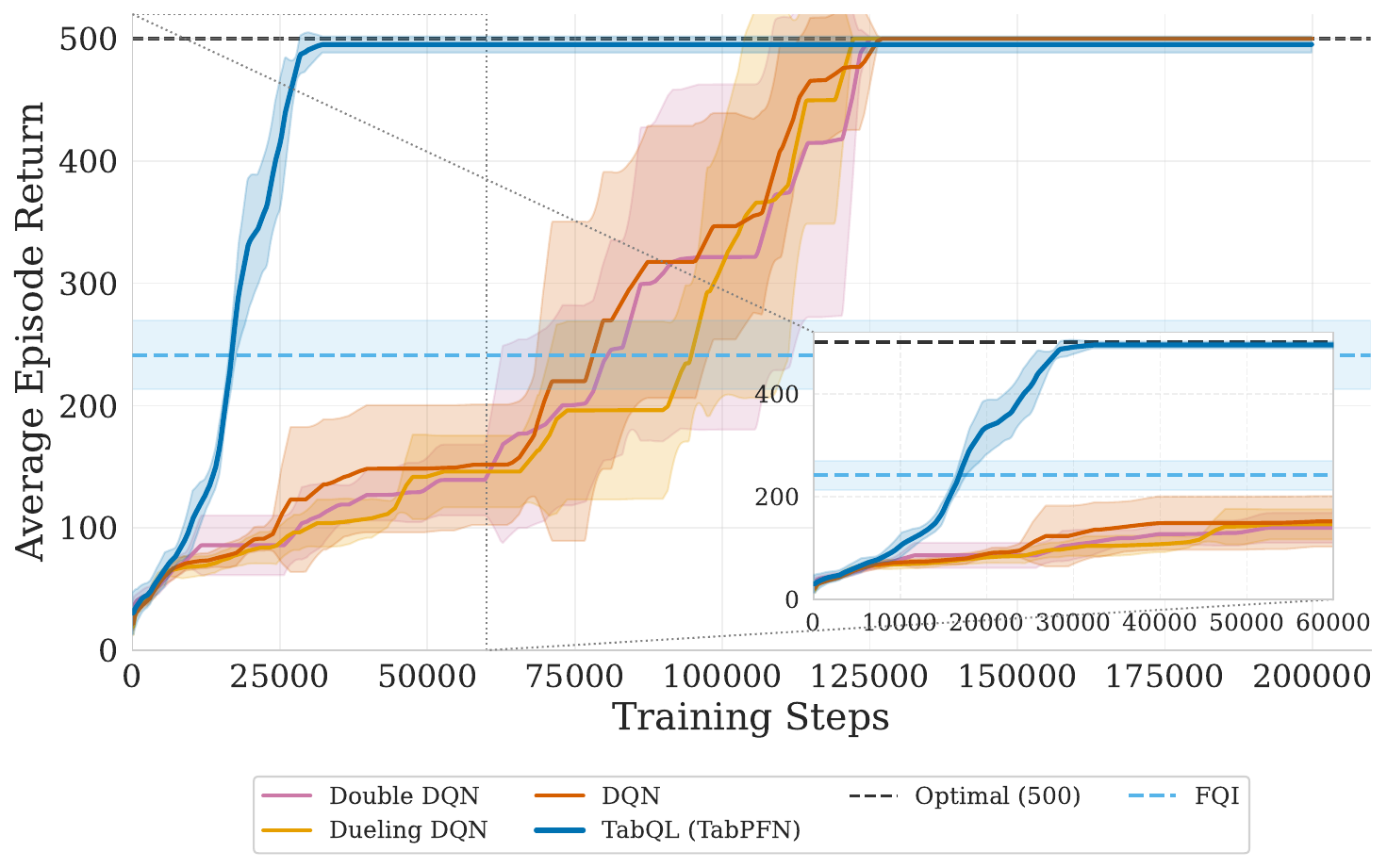}
        \caption{CartPole-v1}
        \label{fig:CartPole_tabQL}
    \end{subfigure}

    \vspace{0.05in}

    \caption{
    TabQL (TabPFN/TabDPT) vs.\ five baselines (Tabular Q, DQN, Double DQN, Dueling DQN, FQI) across three discrete and one continuous-observation environment.
    Solid lines: mean over 5 random seeds; shaded: standard deviation.
    }
    \label{fig:comparative_study}
    \vspace{-0.2in}
\end{figure*}

\textbf{Switching point $T_0$.} Figure~\ref{fig:switching_point} shows TabQL exhibits a clear threshold behavior in $T_0$, as predicted by the early-switching analysis in Section~\ref{sec:theory}. With $T_0$ below the threshold (e.g., 3,000 steps in Taxi), TabQL converges to a suboptimal policy and cannot recover even with continued DQN label correction; once $T_0$ crosses the threshold (e.g., $\geq 5{,}000$ in Taxi), TabQL reliably converges to optimal returns and further increases yield no benefit. The same pattern holds in CliffWalking with a different threshold value, confirming that informative context initialization is essential.

\begin{figure*}[t]
    \centering
    \begin{subfigure}[t]{0.48\textwidth}
        \centering
        \includegraphics[width=\linewidth]{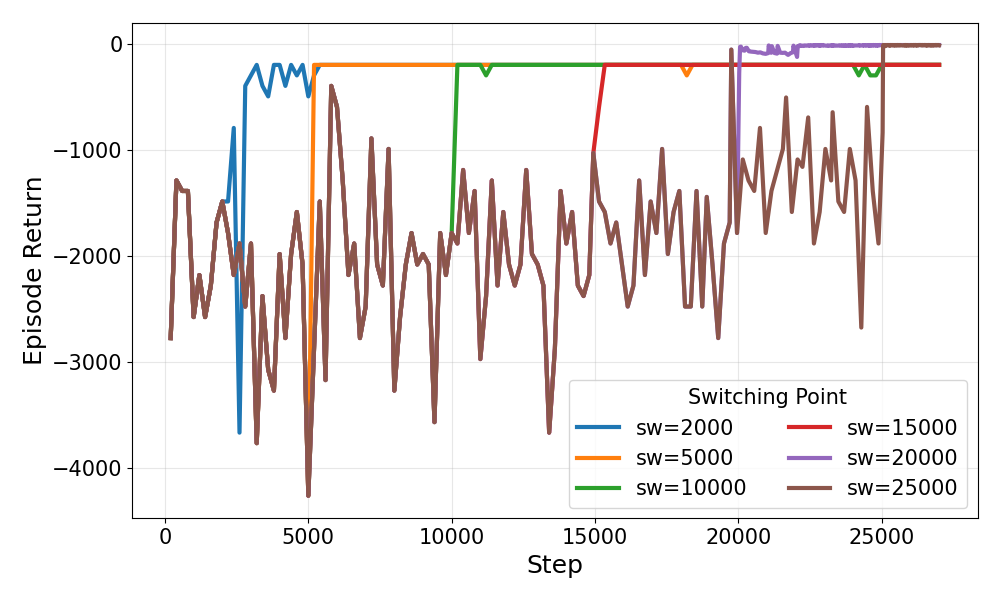}
        \caption{CliffWalking-v1}
        \label{fig:switching_point_cliffwalking}
    \end{subfigure}
    \hfill
    \begin{subfigure}[t]{0.48\textwidth}
        \centering
        \includegraphics[width=\linewidth]{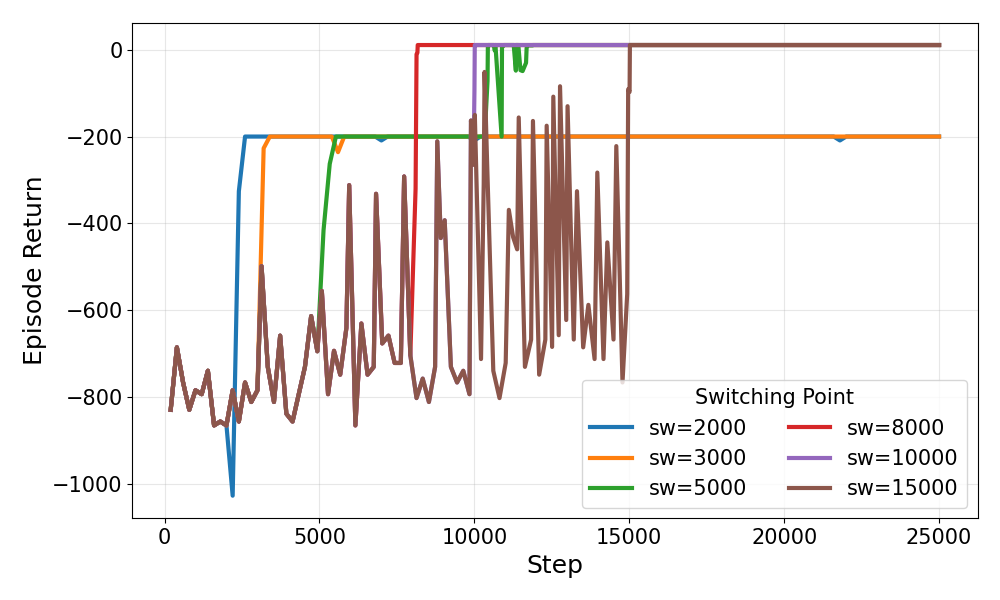}
        \caption{Taxi-v3}
        \label{fig:switching_point_taxi}
    \end{subfigure}
    \caption{
    Switching point analysis: Varying warm-up length $T_0$ reveals a threshold: small $T_0$ yields suboptimal TabQL due to poorly converged DQN (large $\varepsilon_{\mathrm{label}}$); beyond an environment-specific threshold, TabQL consistently reaches the optimal policy, with no gains from further increases.
    }
    \label{fig:switching_point}
    \vspace{-0.2in}
\end{figure*}

\noindent\textbf{Impact of context size.}
We investigate the effect of context size on model performance by evaluating multiple context size settings. This ablation aims to understand how contextual information influences convergence and robustness. Figures~\ref{fig:context_length_cliffwalking} and~\ref{fig:context_length_taxi} illustrate the training progress under different context lengths. In most cases, increasing the context size helps accelerate training and stabilize the learning curve, but only up to a certain point. Beyond this threshold, further increasing the context size yields little to no additional improvement in training behavior. Moreover, in some cases, simply increasing the context size does not lead to better training performance.

\begin{figure*}[t]
    \centering
    \begin{subfigure}[t]{0.48\textwidth}
        \centering
        \includegraphics[width=\linewidth]{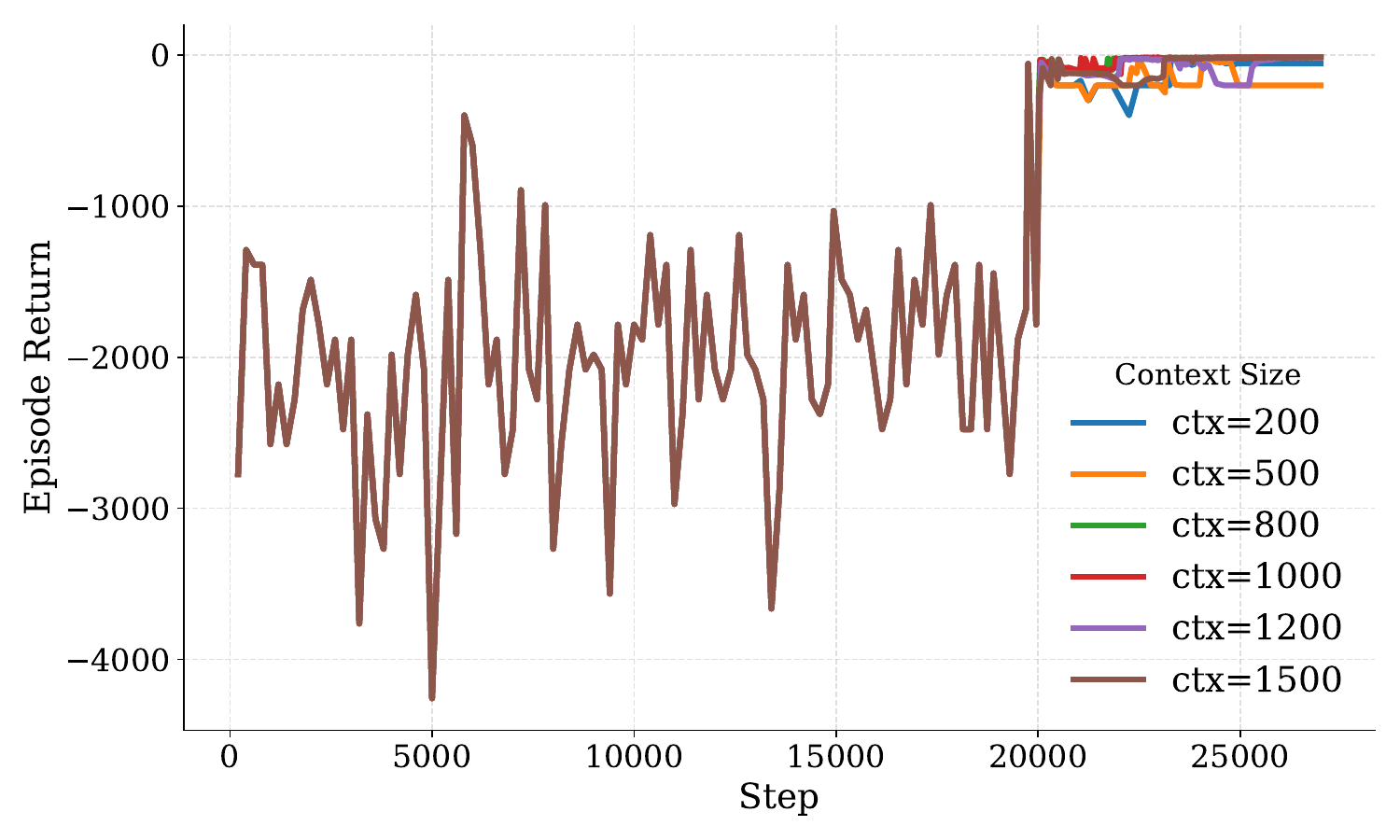}
        \caption{CliffWalking-v1}
        \label{fig:context_length_cliffwalking}
    \end{subfigure}
    \hfill
    \begin{subfigure}[t]{0.48\textwidth}
        \centering
        \includegraphics[width=\linewidth]{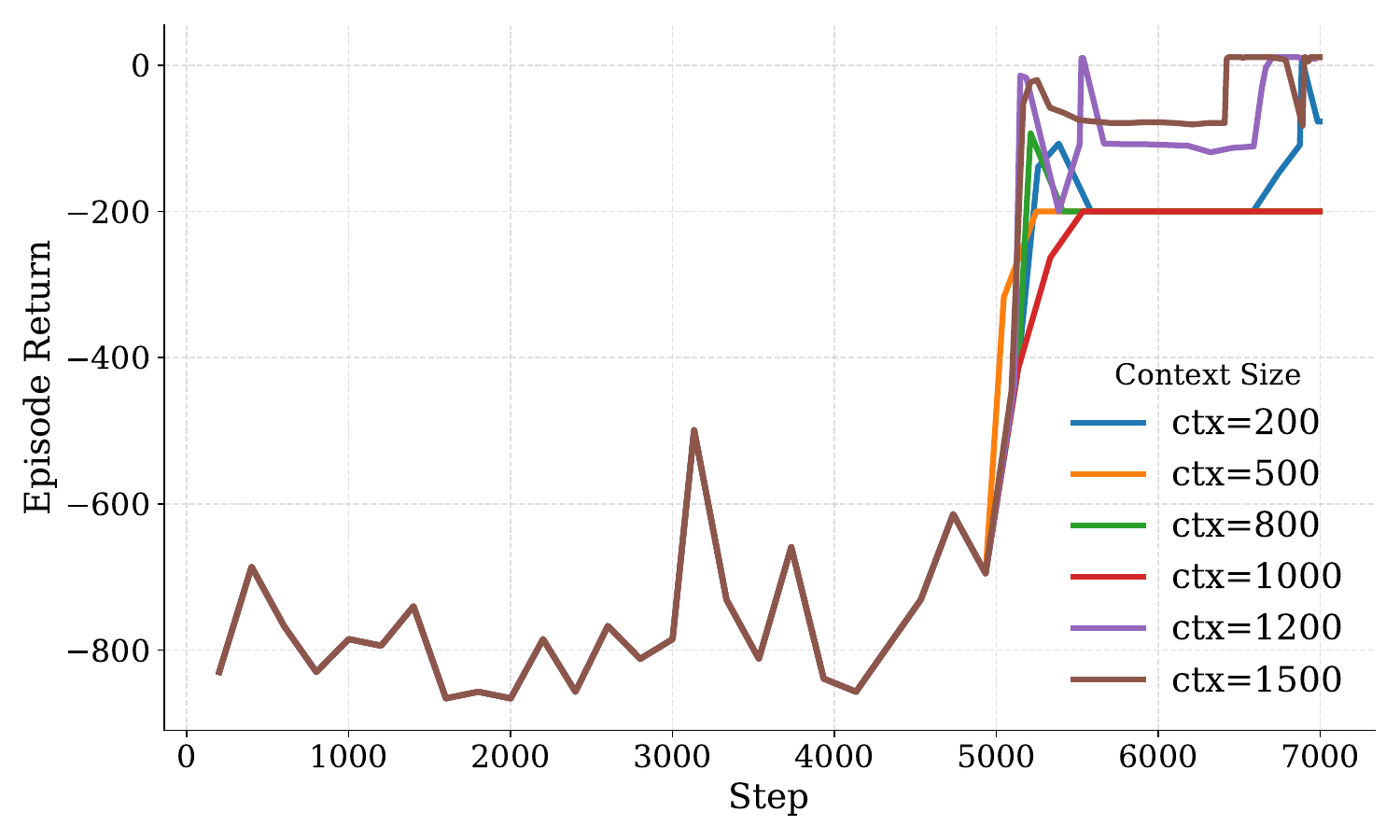}
        \caption{Taxi-v3}
        \label{fig:context_length_taxi}
    \end{subfigure}
    \caption{Impact of context size $K$ on TabQL training. Increasing $K$ accelerates training and stabilizes the learning curve up to a saturation point, beyond which further increases yield diminishing returns.}
    \label{fig:context_size}
    \vspace{-0.2in}
\end{figure*}

\textbf{Generalization.} We evaluate generalization by training agents under multiple environment initializations and aggregating their experience into a shared context for the TFM. At evaluation, TabQL performs in-context inference on unseen initial conditions. As shown in Figure~\ref{fig:generalization}(a), TabQL generalizes consistently across both seen and unseen initial states with lower variance, while DQN frequently fails to transfer and occasionally collapses entirely. This indicates that aggregating experience across seeds lets TabQL learn transferable value inference rather than memorizing trajectory statistics. Feature-encoding details (state decoding, episode timestep, and the cross-seed initial-state feature) are described in Appendix~\ref{feature_selection}.

\begin{figure*}[t]
    \centering
    \begin{subfigure}[t]{0.48\textwidth}
        \centering
        \includegraphics[width=\linewidth]{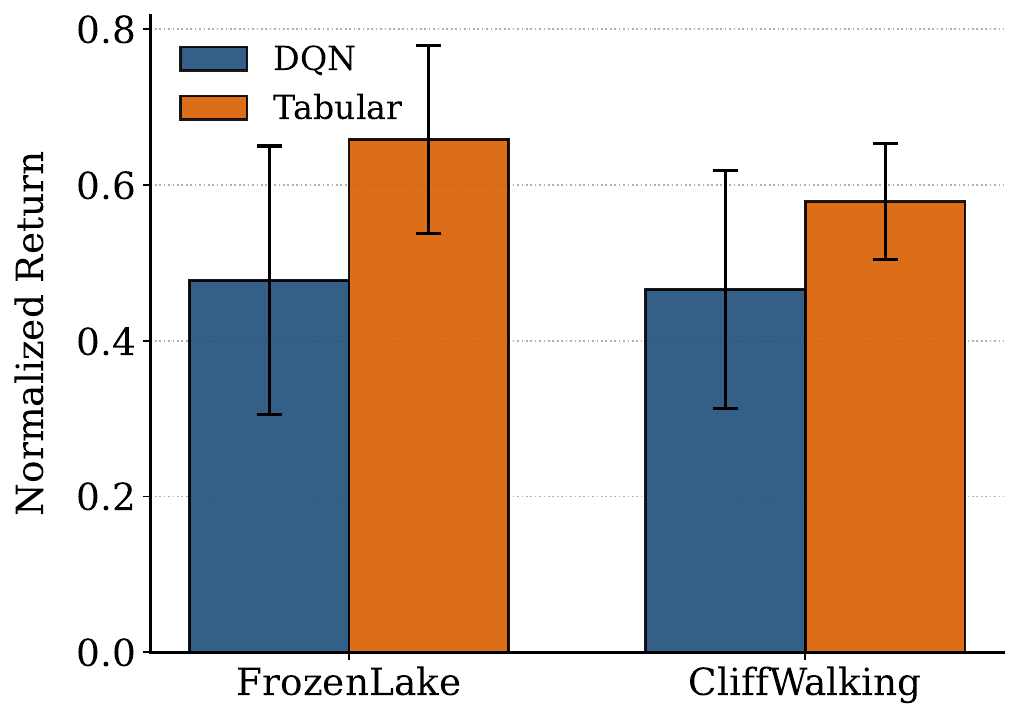}
        \caption{Generalization across environments}
        \label{fig:across_seeds}
    \end{subfigure}
    \hfill
    \begin{subfigure}[t]{0.48\textwidth}
        \centering
        \includegraphics[width=\linewidth]{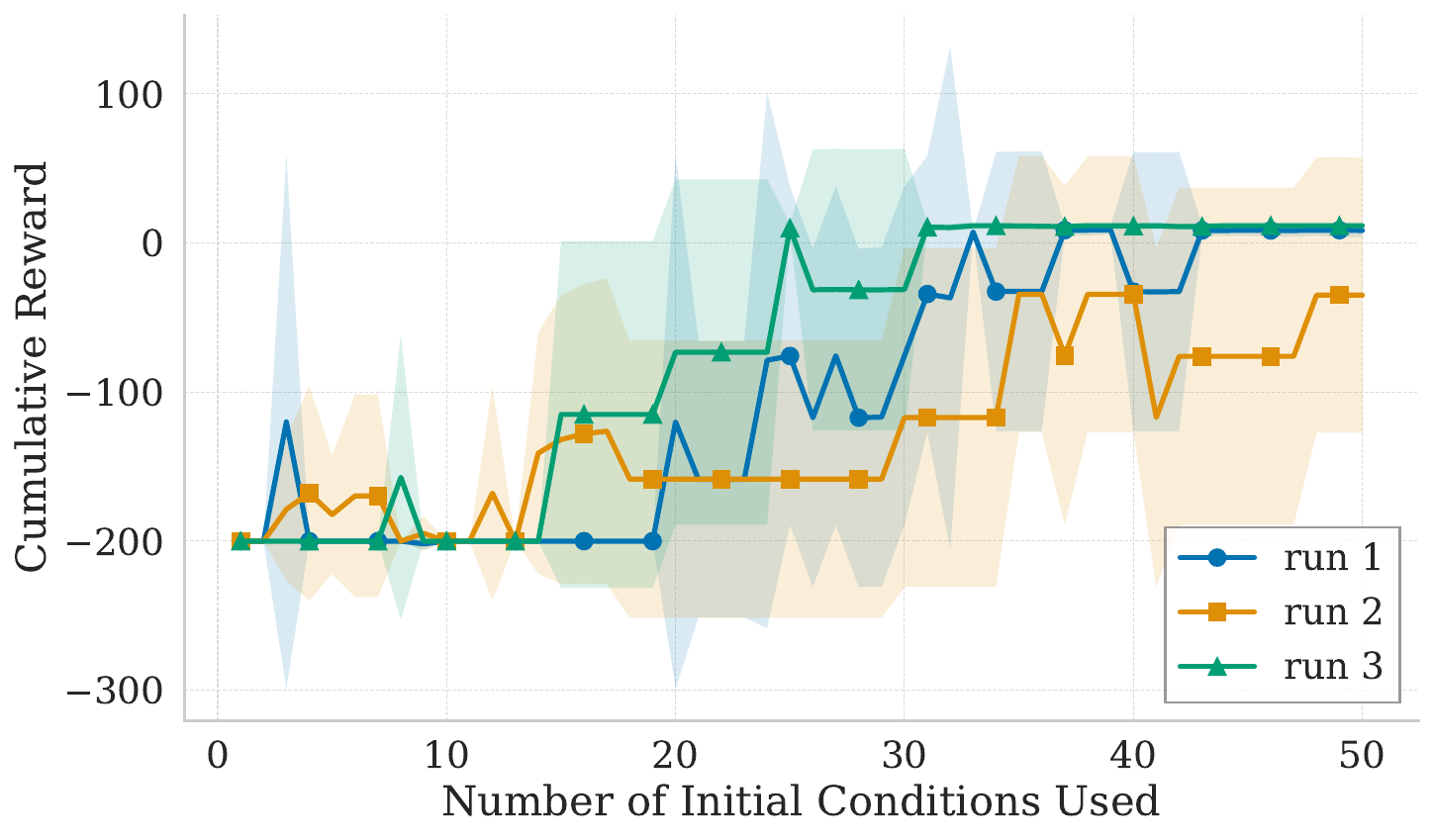}
        \caption{TFM-only generalization on Taxi}
        \label{fig:tab_cross_seeds}
    \end{subfigure}
    \caption{
    (a) Performance across environments: bars show mean normalized return (scaled to [0,1]) on unseen initial states, averaged over independently trained agents.
    (b) TFM-only generalization on Taxi: each run samples 50 of 100 initial conditions; the x-axis is the number included in the TFM context, evaluated on 5 unseen conditions using mean cumulative return.
    }
    \label{fig:generalization}
    \vspace{-0.15in}
\end{figure*}

\textbf{Effectiveness of TFM.}
We perform an additional ablation to isolate the generalization ability of the tabular foundation model (TFM), independent of the full TabQL pipeline. Instead of end-to-end training, we train DQN agents under 100 random initial conditions and randomly select 50 converged agents to collect optimal trajectories for fitting the TFM. The learned TFM is then evaluated on 5 unseen seeds not used during context construction. Figure~\ref{fig:tab_cross_seeds} shows that as the number of seed trajectories in the context increases, performance on all test seeds improves and stabilizes. Once sufficient seed diversity is included, the TFM consistently achieves strong performance across unseen environments, indicating that diverse contextual experience enables the model to learn robust state–action patterns and generalize beyond training seeds.

\textbf{Limitations.} A key limitation is identifying when to switch to in-context learning. Candidate metrics were noisy or environment-dependent, so the switching point remains a tunable hyperparameter. Second, the tabular foundation model suits low-dimensional state–action spaces (e.g., grid worlds). Extending TabQL to high-dimensional inputs like images or large continuous spaces will require more expressive models, a promising future direction.

\section{Conclusion and Broader Impact}
\vspace{-0.1in}
We proposed TabQL, a framework reframing action-value learning as in-context Bellman inference using tabular foundation models. By amortizing Bellman updates into pretrained inference and decoupling learning from online optimization, TabQL improves sample efficiency, stability, and theoretical structure, with convergence and sample-complexity guarantees. A short DQN warm-up provides informative context, and we identified and theoretically justified a failure mode from premature switching. TabQL bridges classical Q-learning, deep RL, and foundation models, pointing to adaptive switching and richer context as future directions. By reducing online-data demands, TabQL benefits costly real-world domains (robotics, healthcare, autonomous systems), though computational cost, pretraining quality, and bias considerations remain essential for safe deployment.

\clearpage
\nocite{langley00}

\bibliography{reference}
\bibliographystyle{unsrt}

\newpage
\appendix
\section{Appendix}
In this section, we present additional related work and analysis, technical discussion, proofs of the main results, and additional experimental results for our proposed TabQL. We start with the additional related work.

\subsection{Additional Related Work}\label{additional_related_work}
Several recent works have explored amortizing expensive planning or search procedures into learned value functions. Notably, SAVE~\cite{hamrick2019combining} amortizes Monte Carlo Tree Search (MCTS)–derived Q-value estimates into a parametric Q-function, leveraging search as a supervisory signal to improve sample efficiency and stability. This approach demonstrates that transferring computationally intensive reasoning into a reusable estimator can significantly accelerate reinforcement learning. TabQL shares this high-level motivation but differs fundamentally in mechanism: rather than amortizing search-derived labels, TabQL amortizes Bellman inference itself through in-context learning. As a result, TabQL does not rely on explicit planning or search rollouts, and its error profile is governed by contextual approximation and statistical coverage rather than the quality of search-generated targets. This distinction highlights complementary strengths: while SAVE inherits robustness from explicit planning, TabQL emphasizes flexibility and generalization via pretrained inference models. A systematic comparison between amortized search and amortized Bellman inference remains an interesting direction for future work. 

The broader perspective of replacing repeated optimization with learned inference has been extensively discussed in the amortized optimization literature~\cite{amos2023tutorial}. In this view, a model is trained to approximate the solution map of an optimization problem, trading iterative computation for fast inference at test time. TabQL fits naturally within this framework by treating Bellman optimality as an implicit optimization problem whose solution, namely the optimal Q-function, is approximated via amortized in-context inference. Unlike classical amortized optimization approaches that rely on explicit supervised training over solution instances, TabQL leverages the emergent in-context learning capabilities of foundation models, enabling adaptation without explicit re-optimization. Making this connection explicit clarifies that TabQL is not merely a new RL algorithm, but an instantiation of amortized optimization principles applied to reinforcement learning dynamics.

Subsequent work on sequence-based models has extended these ideas to hybrid or online regimes and to large discrete action spaces, including methods such as BraVE~\cite{landers2024brave} which combines value estimation and sequence modeling to scale decision-making. While these approaches share TabQL’s use of sequence models, their objectives differ substantially: they typically model policies or returns directly and do not enforce Bellman consistency or Q-function optimality. In contrast, TabQL preserves the structure of Q-learning and explicitly targets Bellman inference through contextual conditioning. This distinction is crucial for understanding TabQL’s theoretical guarantees and its ability to interpolate between vanilla Q-learning, DQN, and in-context learning paradigms.
\subsection{Justification on the Statistical Error}\label{justification}
Our stated sup-norm bound $\varepsilon_{stat}=\mathcal{O}(\sqrt{\frac{\text{log}(|\mathcal{S}||\mathcal{A}|/\delta)}{K}})$ relies on the TFM’s ability to generalize across state–action pairs. More formally, let $\mathcal{F}$ denote the function class induced by the pretrained TFM architecture. If $Q^*\in\mathcal{F}$ (realizability) and the covering number $\mathcal{N}(\mathcal{F},\epsilon)$ is sufficiently small, standard uniform convergence arguments yield sup-norm error scaling as $\sqrt{\text{log}\mathcal{N}(\mathcal{F},\epsilon)/K}$. In practice, the TFM shares parameters across states and actions, so $\text{log}\mathcal{N}(\mathcal{F},\epsilon)\ll|\mathcal{S}||\mathcal{A}|$. Moreover, the rolling context is sampled from a concentrability-covered distribution, reducing the effective number of state–action pairs for which high-accuracy estimation is required. Together, these factors justify the $\sqrt{\frac{\text{log}(|\mathcal{S}||\mathcal{A}|/\delta)}{K}}$ dependence without requiring enumeration over the entire product space.
\subsection{Remark on stochasticity and error propagation}\label{stochasticity}
While the tabular foundation model produces deterministic Q-value estimates conditioned on a fixed context, the context itself is a random object generated by a non-stationary policy. As a result, statistical error is not eliminated, but its role differs fundamentally from that in stochastic-gradient-based Q-learning. In TabQL, sampling noise enters through the construction of an approximate Bellman operator induced by the context and remains fixed during subsequent inference, whereas in DQN, stochasticity enters at every update step and accumulates through repeated stochastic approximation. Our analysis captures this distinction by treating the context-induced operator as a biased but stationary approximation, whose error is contracted rather than iteratively amplified. Although this abstraction does not exactly mirror the on-policy context construction used in practice, it reflects the empirical regime in which the policy and context distribution evolve on a slower timescale than Bellman contraction, a standard assumption in approximate dynamic programming.

\subsection{Methodological Comparison}

We summarize the methodological comparison between DQN and TabQL in Table~\ref{table:comparison} for more detailed reference. The comparison highlights a fundamental shift in how learning is performed in RL. DQN relies on repeated stochastic optimization to approximate the Bellman operator, leading to high interaction cost, accumulated statistical error, and sensitivity to training dynamics. In contrast, TabQL moves the core learning burden into pretrained inference, using context to execute Bellman-style reasoning without iterative updates. This decoupling of learning from online interaction yields substantially improved sample efficiency and stability, at the cost of moderately higher per-step computation and memory usage. As a result, TabQL is particularly well suited for interaction-limited settings such as real-world systems or expensive simulators, where reducing environment interaction is more critical than minimizing inference cost, while DQN remains preferable in regimes where interaction is cheap and computation is constrained. 

\begin{table}[t]
  \caption{Methodological comparison of TabQL with the baselines used in our experiments. Double DQN and Dueling DQN inherit DQN's column except for known architectural refinements.}
  \label{table:comparison}
  \begin{center}
    \scriptsize
    \setlength{\tabcolsep}{4pt}
      \begin{tabular}{lcccc}
        \toprule
        Aspect & Tabular Q & DQN / Double / Dueling & FQI & TabQL \\
        \midrule
        Bellman update & Tabular update & SGD on parameters & Batch fitting & In-context inference \\
        Adaptation mechanism & Table entries & Gradient updates & Batch refit & Context construction \\
        Online vs.\ batch & Online & Online & Batch (offline) & Online \\
        Per-step error source & TD sample noise & Gradient noise $\sigma^2$ & Approximation error & $\varepsilon_{\mathrm{ICL}} + \varepsilon_{\mathrm{label}}$ \\
        Error accumulation & $\sum_t \alpha_t^2$ & $\sum_t \alpha_t^2$ (compounds) & Per-iteration & Bounded constants \\
        Sample complexity & $\tilde{\mathcal{O}}(\frac{|\mathcal{S}||\mathcal{A}|}{(1-\gamma)^4\epsilon^2})$ & $\tilde{\mathcal{O}}(\frac{|\mathcal{S}||\mathcal{A}|}{(1-\gamma)^4\epsilon^2})$ & Dataset-dependent & $\tilde{\mathcal{O}}(\frac{c \cdot m_{\min}}{(1-\gamma)^4\epsilon^2})$ \\
        Function approximation & \xmark & \cmark & \cmark & \cmark \\
        Uses pretrained model & \xmark & \xmark & \xmark & \cmark  \\
        Parameter updates during RL & Table only & SGD & Batch & None  \\
        Per-step compute & Very low & Low (forward + backward) & High (batch refit) & Higher (context inference) \\
        Memory footprint & $|\mathcal{S}||\mathcal{A}|$ table & Parameters & Full dataset & Parameters + context \\
        \bottomrule
      \end{tabular}
  \end{center}
  \vskip -0.1in
\end{table}

\subsection{Adaptive Switching Criterion}\label{app:adaptive_switch}

The switching point $T_0$ in Algorithm~\ref{alg:fql} is treated as a tunable hyperparameter in the main text. Here we describe a quality-gated switching rule that arms at a soft floor and fires only once recent DQN episodes demonstrate sufficient quality, mitigating the early-switching failure characterized in Section~\ref{sec:theory}.

After $t \geq T_0$, let $\mathcal{W}_t = \{R_{t-W+1}, \ldots, R_t\}$ denote the rolling window of the last $W$ episode returns ($W=30$ in our experiments). Define the quality threshold $\theta_t = \mathrm{Quantile}_q(\mathcal{W}_t)$ and the count of episodes above threshold $G_t = |\{R \in \mathcal{W}_t : R > \theta_t\}|$. The switch fires when
\begin{equation}\label{eq:switch_gate}
G_t \geq G_{\min} \quad\text{and}\quad \theta_t > \theta_{\mathrm{floor}} + \delta,
\end{equation}
where $G_{\min}=20$, $\theta_{\mathrm{floor}}$ is an environment-specific failure floor, and $\delta=1$ is a margin against degenerate quantiles. We use $q=0.5$ for dense-reward environments and $q=0.75$ for sparse-reward environments where the median return collapses to the failure floor early in training.

\textbf{Adaptive switch timing for sparse-reward environments. } When DQN may not find non-failure trajectories before $T_0$, the gate alone is insufficient: $\theta_t$ would equal the failure floor, and $R > \theta_t$ would admit failed episodes. We therefore allow the actual switching step to lag behind the configured $T_0$. While Eq.~\ref{eq:switch_gate} remains unmet, DQN continues to receive gradient updates. This lets the agent escape the failure floor before any context is committed. The handoff to the TFM occurs as soon as Eq.~\ref{eq:switch_gate} is first satisfied, capturing a non-degenerate context pool.

\subsection{Adaptive Refit Frequency}\label{app:adaptive_refit}

After the switch, the TFM must periodically refit its context to incorporate new transitions. A fixed refit cadence (every $N$ episodes) is brittle. The short-episode environments accumulate buffer data faster per episode than long-episode ones, so a value tuned for one environment over- or under-fits in another.

We instead refit when the context becomes stale relative to its size. Let $\Delta_t = t - t_{\mathrm{last}}$ be the number of environment steps since the last refit and $K$ the context size. The staleness fraction is $\rho_t = \Delta_t / K$. A refit triggers when
\begin{equation}\label{eq:refit_gate}
\rho_t \geq \rho_{\mathrm{stale}} \quad\text{and}\quad e_t - e_{\mathrm{last}} \geq e_{\min},
\end{equation}
with $\rho_{\mathrm{stale}} = 0.25$ and $e_{\min}=1$. Intuitively, refit once at least 25\% of the context has turned over since the last fit. This formulation is environment-agnostic: a fast environment collecting 4{,}000 transitions per episode and a slow one collecting 100 both refit at the same fractional turnover, which is the quantity that controls how much new information the context could absorb.

\subsection{In-Context Bellman Inference}\label{app:icl_bellman_full}

We provide the full conceptual development of TabQL's in-context Bellman inference, summarized in Section~\ref{sec:tabql} of the main body.

\textbf{Conditional inference perspective.} Recall the optimal action-value function $Q^*$ satisfies $Q^* = \mathcal{T}Q^*$, and classical tabular Q-learning approximates this fixed point via the stochastic update
\begin{equation*}
Q_{t+1}(s_t,a_t)=(1-\alpha_t)Q_t(s_t,a_t)+\alpha_t\left(r_t+\gamma\max_{a'}Q_t(s_{t+1},a')\right).
\end{equation*}
Each transition contributes a single local constraint, and global consistency emerges only after many updates and extensive environment interaction. TabQL adopts a different viewpoint: rather than iteratively reducing per-step Bellman residuals, we treat Q-value estimation as a conditional inference problem on a finite context of recent experience. The context $\mathcal{C}_t$ defines an empirical MDP, and the TFM is asked to infer the Q-value of any query pair $(s,a)$ given the value structure encoded in $\mathcal{C}_t$.

\textbf{TFM as a sequence model.} The TFM is a sequence model with $f_\phi(\mathcal{C}_t, s, a) = f_\phi(\mathcal{C}_t \oplus (s,a))$, where $\oplus$ denotes concatenation of the query with the context. Two key properties make this useful in TabQL: (i) \emph{tabularized input} — each transition in $\mathcal{C}_t$ is encoded as a discrete or discretized tuple $(s, a, r, s', \hat{Q}^{\mathrm{DQN}}(s,a))$, allowing the model to generalize tabular reasoning across state-action pairs; (ii) \emph{in-context adaptation} — conditioning on $\mathcal{C}_t$ enables zero- or few-shot adjustment to the empirical dynamics encountered online, without any gradient updates to $\phi$. Importantly, the TFM does not maintain an explicit Q-table; Q-values are computed on demand via in-context inference.

\textbf{Why pretraining on tabular regression suffices.} A natural concern is whether the TFM must be pretrained specifically on tabular RL distributions to approximate Bellman reasoning. Our results indicate this is not necessary. TabPFN and TabDPT are pretrained on millions of synthetic tabular regression tasks, encoding a strong inductive bias toward smooth, low-complexity tabular functions. When given a context of $(s, a, \hat{Q}^{\mathrm{DQN}}(s,a))$ triples drawn from a converging warm-up policy, the TFM performs amortized regression — effectively averaging over prior-compatible smooth functions consistent with the context. This produces Q-estimates with lower variance than the raw DQN labels, even when those labels are individually noisy, because the smoothness prior implicitly denoises across nearby state-action pairs.

\textbf{Mechanism summary.} The mechanism that makes TabQL work is therefore not Bellman-specific pretraining, but rather the TFM's general-purpose tabular regression capability applied to a context whose labels happen to be (noisy) Bellman targets. This shifts the burden of generalization from gradient-based parameter updates to one-shot inference — the central architectural advantage TabQL inherits from in-context learning.

\subsection{Tabular Foundation Model}\label{app:tfm_full}

The tabular foundation model (TFM) used in TabQL is a parameterized sequence model
\begin{equation*}
f_\phi: \mathcal{X}^K \times (\mathcal{S}\times\mathcal{A}) \to \mathbb{R},
\end{equation*}
that maps a context of $K$ tabularized transitions and a query state-action pair to a scalar Q-value prediction. The induced contextual Q-function is $\hat{Q}^{\mathrm{TFM}}(s,a|\mathcal{C}) := f_\phi(\mathcal{C}, s, a)$.

\textbf{Pretraining.} The parameters $\phi$ are pretrained on a distribution of synthetic tabular regression tasks (e.g., TabPFN~\cite{hollmann2022tabpfn}, TabDPT~\cite{ma2024tabdpt}). This pretraining is task-independent: it does not assume access to RL trajectories or Bellman targets, and it is not modified by TabQL at any point. The TFM is used purely as an off-the-shelf in-context regressor.

\textbf{Adaptation through context, not gradients.} During TabQL's online learning, adaptation occurs entirely through conditioning on $\mathcal{C}_t$. The TFM weights $\phi$ are frozen for the duration of every TabQL run; no fine-tuning is performed. Each new transition appended to the buffer immediately influences subsequent Q-estimates by virtue of being included in $\mathcal{C}_t$, without waiting for gradient propagation through any parameters.

\textbf{Greedy policy.} Given context $\mathcal{C}_t$ at time $t$, TabQL selects actions according to the greedy policy
\begin{equation*}
\pi_{\phi,\mathcal{C}_t}(s) \in \arg\max_{a\in\mathcal{A}} \hat{Q}^{\mathrm{TFM}}_t(s,a).
\end{equation*}
The full TabQL action-selection rule includes $\varepsilon$-greedy exploration on top of this greedy policy (see Algorithm~\ref{alg:fql}).

\textbf{TFM backbones used.} We evaluate TabQL with two TFM backbones: TabPFN (v0.1.9, tree-based with attention readout) and TabDPT (transformer-based). Both achieve similar performance on our benchmarks (Section~\ref{sec:experiments}), suggesting the result is robust to the specific TFM choice.

\subsection{Asymptotic Behavior and Performance Loss Derivation}\label{asymptotic_derivation}

We provide the full derivation of the asymptotic bound and performance loss referenced in Section~\ref{sec:theory}. As the buffer grows and Assumption~\ref{assumption_3} holds with $m_t \to m_{\min}$ uniformly, the statistical error converges to a finite constant
\begin{equation*}
\bar{\varepsilon}_{\mathrm{stat}} \;\leq\; \varepsilon_{\mathrm{label}} + \frac{\gamma}{1-\gamma}\sqrt{\frac{2\log(|\mathcal{S}||\mathcal{A}|/\delta)}{m_{\min}}}.
\end{equation*}
By the dominated convergence theorem, since $\sum_{\tau=0}^\infty \gamma^\tau = (1-\gamma)^{-1} < \infty$, the discounted sum in Theorem~\ref{theorem_1} converges to a bounded limit:
\begin{equation*}
\limsup_{t\to\infty}\|\hat{Q}^{\mathrm{TFM}}_t - Q^*\|_\infty \;\leq\; \frac{\varepsilon_{\mathrm{ICL}} + \varepsilon_{\mathrm{label}}}{1-\gamma} + \frac{\gamma}{(1-\gamma)^2}\sqrt{\frac{2\log(|\mathcal{S}||\mathcal{A}|/\delta)}{m_{\min}}}.
\end{equation*}
The performance-difference inequality~\cite{antos2008learning,munos2008finite} gives $V^*(s) - V^{\pi_t}(s) \leq \frac{2\gamma}{1-\gamma}\|\hat{Q}^{\mathrm{TFM}}_t - Q^*\|_\infty$. Combining with the asymptotic bound yields the asymptotic suboptimality $V^*(s)-V^{\pi_\infty}(s)$ when $t\to\infty$.

\subsection{Full Conditional Comparison with DQN}\label{dqn_comparison_full}

We expand the discussion in Section~\ref{sec:theory} of how TabQL's error structure differs from DQN's. The standard DQN sample complexity is $N_{\mathrm{DQN}}=\tilde{\mathcal{O}}(|\mathcal{S}||\mathcal{A}|/((1-\gamma)^4\epsilon^2))$~\cite{li2024q}, arising from the stochastic gradient recursion
\begin{equation*}
\mathbb{E}\|\hat{Q}^{\mathrm{DQN}}_{t+1}-Q^*\|^2 \leq (1-\alpha_t(1-\gamma))\mathbb{E}\|\hat{Q}^{\mathrm{DQN}}_t-Q^*\|^2+\alpha_t^2\sigma^2,
\end{equation*}
where each update injects variance $\sigma^2 = \Omega(1/(1-\gamma)^2)$ from the single-sample TD target. The resulting $\sum_t \alpha_t^2$ accumulation is intrinsic to stochastic gradient methods and produces a fundamental $\epsilon^{-2}$ dependence that cannot be reduced without additional structural assumptions.

In TabQL, when Assumptions~\ref{assumption_1}--\ref{assumption_3} hold, Bellman updates are computed via in-context inference rather than gradient steps. Per-step variance is replaced by the fixed approximation error $\varepsilon_{\mathrm{ICL}}$ and the inherited label error $\varepsilon_{\mathrm{label}}$, neither of which compounds across iterations. The discounted sum in Theorem~\ref{theorem_1} takes the form $\sum_{\tau=0}^{t-1}\gamma^\tau \varepsilon_{\mathrm{stat}}(\tau)$ with $\varepsilon_{\mathrm{stat}}(\tau) = \mathcal{O}(1/\sqrt{m_\tau}) + \varepsilon_{\mathrm{label}}$, in contrast to DQN's $\sum_t \alpha_t^2$ variance accumulation. TabQL therefore trades the hyperparameter cost of choosing $T_0$ and $K$ for a structural advantage in error propagation. When Assumption~\ref{assumption_3} is violated or $\varepsilon_{\mathrm{label}}$ is large, this advantage vanishes and TabQL degrades gracefully toward the warm-up DQN's behavior.

\subsection{Mechanistic Analysis of Early-Switching Failure}\label{early_switching_full}

The corrected error bound from Theorem~\ref{theorem_1} formally explains the empirical early-switching failure:
\begin{equation*}
\|\hat{Q}^{\mathrm{TFM}}_t-Q^*\|_\infty\leq \gamma^t\|\hat{Q}^{\mathrm{TFM}}_0-Q^*\|_\infty+\sum_{\tau=0}^{t-1}\gamma^\tau\big(\varepsilon_{\mathrm{ICL}}+\varepsilon_{\mathrm{stat}}(\tau)\big),
\end{equation*}
where $\varepsilon_{\mathrm{stat}}(\tau)$ contains the label error term $\varepsilon_{\mathrm{label}} = \|\hat{Q}^{\mathrm{DQN}} - Q^*\|_\infty$. Both the initial approximation error and $\varepsilon_{\mathrm{label}}$ are determined by the warm-up DQN at the moment of switching.

Crucially, this bias cannot be corrected by additional context or longer ICL inference. TabQL applies Bellman contraction to an approximate operator realized via in-context inference; if the labels supplied to that operator are systematically biased, the iteration converges toward a biased fixed point corresponding to a locally optimal policy rather than $Q^*$. Even when the warm-up DQN continues to collect transitions in parallel, TabQL primarily conditions on recent on-policy context, causing the context distribution to concentrate in a suboptimal region of the state-action space. While the statistical estimation error in $\hat{P}_{\mathcal{C}_t}$ diminishes as the context grows, the bias term $\varepsilon_{\mathrm{label}}$ persists. This behavior mirrors classical results in approximate dynamic programming, where biased function approximation leads Bellman iteration to converge only to a neighborhood of a suboptimal solution rather than $Q^*$.

\subsection{Discussion on Dependence on DQN Labels After Switching}\label{dependence}
While TabQL reduces reliance on repeated gradient-based Bellman optimization, the current instantiation uses DQN as a label generator after the warm-up phase. This does not contradict the amortization objective: DQN serves as a transient teacher that provides approximate Bellman targets, whereas TabQL performs the primary value inference and policy improvement through in-context estimation. The roles are therefore asymmetric as DQN supplies supervision signals, while TabQL replaces iterative optimization with amortized inference. Our theoretical analysis explicitly allows for imperfect teacher labels, showing that TabQL’s value error scales additively with label noise and is subsequently contracted by the Bellman operator. This structure is analogous to amortized optimization and search-distillation frameworks, where expensive or imperfect solvers are used to train fast inference models. Importantly, empirical results indicate that once TabQL reaches sufficient contextual coverage, its performance becomes significantly less sensitive to teacher accuracy, suggesting that the amortized estimator progressively dominates the learning dynamics.

\subsection{Proofs of Main Results}\label{missing_proof}

We first prove Theorem~\ref{theorem_1}.
\begin{proof}
Starting from the error decomposition in Eq.~\eqref{eq_15}:
\begin{equation*}
\hat{Q}^{\mathrm{TFM}}_{t+1} - Q^* = (\mathcal{T}\hat{Q}^{\mathrm{TFM}}_t - \mathcal{T}Q^*) 
+ (\hat{\mathcal{T}}_{\mathcal{C}_t}\hat{Q}^{\mathrm{TFM}}_t - \mathcal{T}\hat{Q}^{\mathrm{TFM}}_t)
+ (\hat{Q}^{\mathrm{TFM}}_{t+1} - \hat{\mathcal{T}}_{\mathcal{C}_t}\hat{Q}^{\mathrm{TFM}}_t).
\end{equation*}
Taking the sup-norm and applying the triangle inequality:
\begin{equation*}
\|\hat{Q}^{\mathrm{TFM}}_{t+1} - Q^*\|_\infty \leq 
\|\mathcal{T}\hat{Q}^{\mathrm{TFM}}_t - \mathcal{T}Q^*\|_\infty 
+ \|\hat{\mathcal{T}}_{\mathcal{C}_t}\hat{Q}^{\mathrm{TFM}}_t - \mathcal{T}\hat{Q}^{\mathrm{TFM}}_t\|_\infty
+ \|\hat{Q}^{\mathrm{TFM}}_{t+1} - \hat{\mathcal{T}}_{\mathcal{C}_t}\hat{Q}^{\mathrm{TFM}}_t\|_\infty.
\end{equation*}
By the standard $\gamma$-contraction property of the Bellman operator $\mathcal{T}$ under sup-norm, the first term is bounded by $\gamma\|\hat{Q}^{\mathrm{TFM}}_t - Q^*\|_\infty$. The second term is $\varepsilon_{\mathrm{stat}}(t)$ by definition. The third term is bounded by $\varepsilon_{\mathrm{ICL}}$ by Assumption~\ref{assumption_1}. Combining,
\begin{equation*}
\|\hat{Q}^{\mathrm{TFM}}_{t+1} - Q^*\|_\infty \leq \gamma\|\hat{Q}^{\mathrm{TFM}}_t - Q^*\|_\infty + \varepsilon_{\mathrm{ICL}} + \varepsilon_{\mathrm{stat}}(t).
\end{equation*}
Unrolling this recursion from $t=0$ yields:
\begin{equation*}
\|\hat{Q}^{\mathrm{TFM}}_t - Q^*\|_\infty \leq \gamma^t\|\hat{Q}^{\mathrm{TFM}}_0 - Q^*\|_\infty + \sum_{\tau=0}^{t-1}\gamma^{t-1-\tau}\big(\varepsilon_{\mathrm{ICL}} + \varepsilon_{\mathrm{stat}}(\tau)\big).
\end{equation*}
Since the geometric weights $\gamma^{t-1-\tau}$ form the same partial sum as $\gamma^\tau$ when $\varepsilon_{\mathrm{stat}}(\tau)$ is bounded uniformly, we obtain the form stated in Theorem~\ref{theorem_1}. The bound on $\varepsilon_{\mathrm{stat}}(\tau)$ follows from Hoeffding's inequality applied to the empirical kernel under Assumption~\ref{assumption_3}, combined with the label error decomposition (derived in the proof of Theorem~\ref{theorem_2} below).
\end{proof}

We now prove Theorem~\ref{theorem_2}.
\begin{proof}
Since rewards are bounded in $[0,1]$, $\|\hat{Q}^{\mathrm{TFM}}_t\|_\infty \leq 1/(1-\gamma)$. Using Hoeffding's inequality for the empirical kernel and a union bound over $|\mathcal{S}||\mathcal{A}|$ pairs, the statistical error decomposes as
\begin{equation*}
\varepsilon_{\mathrm{stat}}(t) \leq \frac{\gamma}{1-\gamma}\sqrt{\frac{2\log(|\mathcal{S}||\mathcal{A}|/\delta)}{m_t}} + \varepsilon_{\mathrm{label}},
\end{equation*}
where $m_t \geq m_{\min}$ is the per-query revisitation count guaranteed by Assumption~\ref{assumption_3}. The label term $\varepsilon_{\mathrm{label}}$ is fixed at the end of the warm-up phase and does not depend on $t$. Recall the recursion from Theorem~\ref{theorem_1}:
\begin{equation*}
\|\hat{Q}^{\mathrm{TFM}}_t-Q^*\|_\infty \leq \gamma^t\|\hat{Q}^{\mathrm{TFM}}_0-Q^*\|_\infty + \sum_{\tau=0}^{t-1}\gamma^\tau\big(\varepsilon_{\mathrm{ICL}}+\varepsilon_{\mathrm{stat}}(\tau)\big).
\end{equation*}
We bound each contribution to keep the total below $\epsilon$.

\textit{Initial error.} Since $\|\hat{Q}^{\mathrm{TFM}}_0-Q^*\|_\infty \leq 1/(1-\gamma)$, requiring $\gamma^t/(1-\gamma) \leq \epsilon/3$ yields
\begin{equation*}
t \geq \frac{\log(3/((1-\gamma)\epsilon))}{\log(1/\gamma)} = \mathcal{O}\!\left(\frac{1}{1-\gamma}\log\frac{1}{\epsilon}\right).
\end{equation*}

\textit{ICL error.} The geometric sum gives $\sum_{\tau=0}^{t-1}\gamma^\tau \varepsilon_{\mathrm{ICL}} \leq \varepsilon_{\mathrm{ICL}}/(1-\gamma)$. Requiring this to be at most $\epsilon/3$ gives the hypothesis $\varepsilon_{\mathrm{ICL}} \leq (1-\gamma)\epsilon/3$, which is a model-capacity condition rather than an interaction requirement.

\textit{Label error.} Similarly, $\sum_{\tau=0}^{t-1}\gamma^\tau \varepsilon_{\mathrm{label}} \leq \varepsilon_{\mathrm{label}}/(1-\gamma)$. The hypothesis $\varepsilon_{\mathrm{label}} \leq (1-\gamma)\epsilon/3$ controls this contribution. This is a warm-up quality condition: longer warm-up reduces $\varepsilon_{\mathrm{label}}$.

\textit{Statistical error from finite context.} For the remaining contribution, fixing $m_\tau = m \geq m_{\min}$,
\begin{equation*}
\sum_{\tau=0}^{t-1}\gamma^\tau \cdot \frac{\gamma}{1-\gamma}\sqrt{\frac{2\log(|\mathcal{S}||\mathcal{A}|/\delta)}{m}} \leq \frac{\gamma}{(1-\gamma)^2}\sqrt{\frac{2\log(|\mathcal{S}||\mathcal{A}|/\delta)}{m}}.
\end{equation*}
To make this at most $\epsilon/3$, we require
\begin{equation*}
m \geq \frac{18\gamma^2}{(1-\gamma)^4\epsilon^2}\log\!\left(\frac{|\mathcal{S}||\mathcal{A}|}{\delta}\right).
\end{equation*}

Since each environment interaction adds one transition to the buffer, achieving per-query revisitation $m$ requires (at least) $N \geq c \cdot m$ total interactions, where the constant $c$ depends on the visitation distribution. Uniform exploration gives $c = \Theta(|\mathcal{S}||\mathcal{A}|)$, whereas more concentrated exploration that prioritizes the relevant region of the state-action space gives a substantially smaller $c$. The interaction cost from this term is therefore
\begin{equation*}
N_{\mathrm{stat}} = \tilde{\mathcal{O}}\!\left(\frac{c}{(1-\gamma)^4\epsilon^2}\log\frac{|\mathcal{S}||\mathcal{A}|}{\delta}\right).
\end{equation*}

Combining the contraction horizon and the statistical cost gives the total
\begin{equation*}
N = \tilde{\mathcal{O}}\!\left(\frac{1}{(1-\gamma)^4\epsilon^2}\log\frac{|\mathcal{S}||\mathcal{A}|}{\delta} + \frac{1}{1-\gamma}\log\frac{1}{\epsilon}\right),
\end{equation*}
environment interactions, conditional on the hypotheses $\varepsilon_{\mathrm{ICL}} \leq (1-\gamma)\epsilon/3$ and $\varepsilon_{\mathrm{label}} \leq (1-\gamma)\epsilon/3$.
\end{proof}

\subsection{Feature Selection}\label{feature_selection}
We evaluated several feature representations for the tabular in-context learning model. Our results indicate that representing each transition using the (state, action) pair is both sufficient and effective. In contrast, directly using the raw discrete state index provided by the environment did not lead to meaningful performance improvements. Instead, we decode the state index back into its original environment-specific representation using the environment’s native decoding function, which yields more informative and structured inputs for the tabular model.

Beyond state and action features, we found that incorporating the episode timestep significantly improves performance. The timestep provides essential temporal context, particularly in our grid-world finite-horizon environments, where optimal actions depend on minimizing the number of steps taken. Accordingly, we include the episode timestep as part of each transition stored in the replay buffer. 
For cross-seed evaluation, we include the initial state as an additional feature, allowing the model to distinguish trajectories from different seed initializations. We observe substantial performance improvement after including this feature.

\subsection{Additional Results}
\begin{figure}[h]
    \centering
    \includegraphics[width=0.5\linewidth]{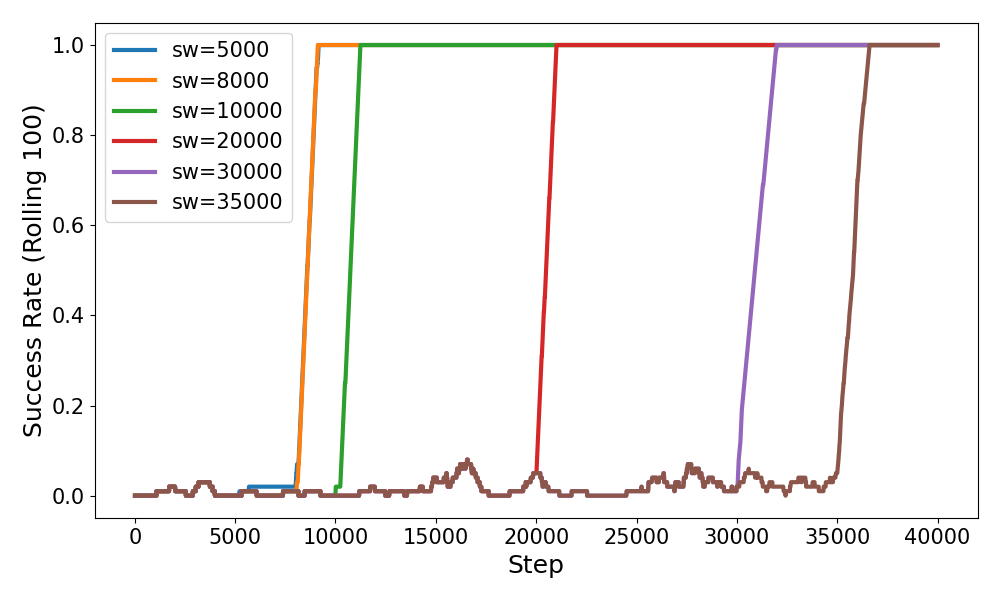}
    \caption{Switching point analysis for Frozen Lake}
    \label{fig:frozenlake_switching_plot}
\end{figure}

\begin{figure}[h]
    \centering
    \includegraphics[width=0.5\linewidth]{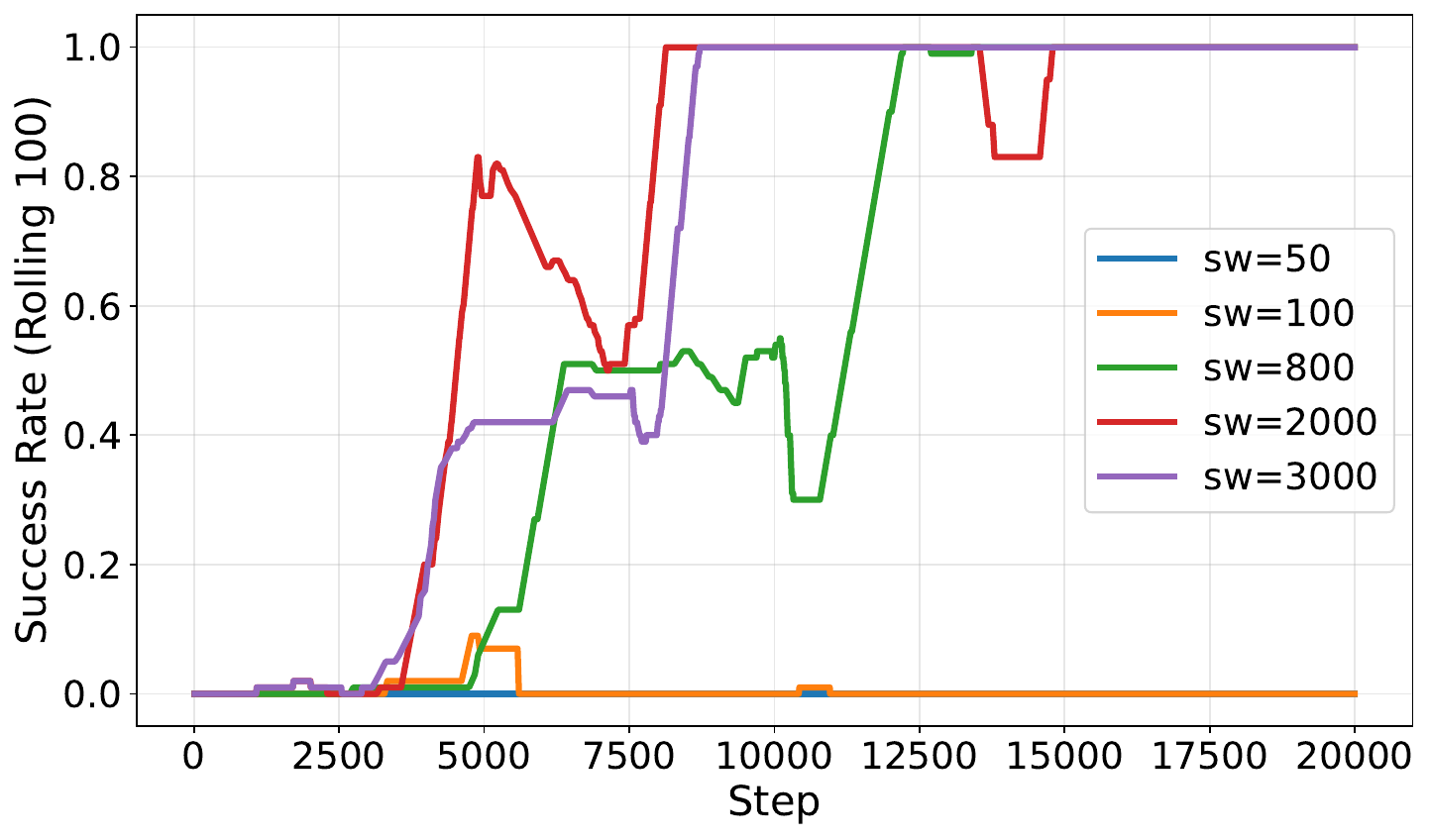}
    \caption{Switching point analysis under extremely early switching in FrozenLake.}
    \label{fig:frozenlake_switching_plot_low}
\end{figure}

\begin{figure}[h]
    \centering
    \includegraphics[width=0.5\linewidth]{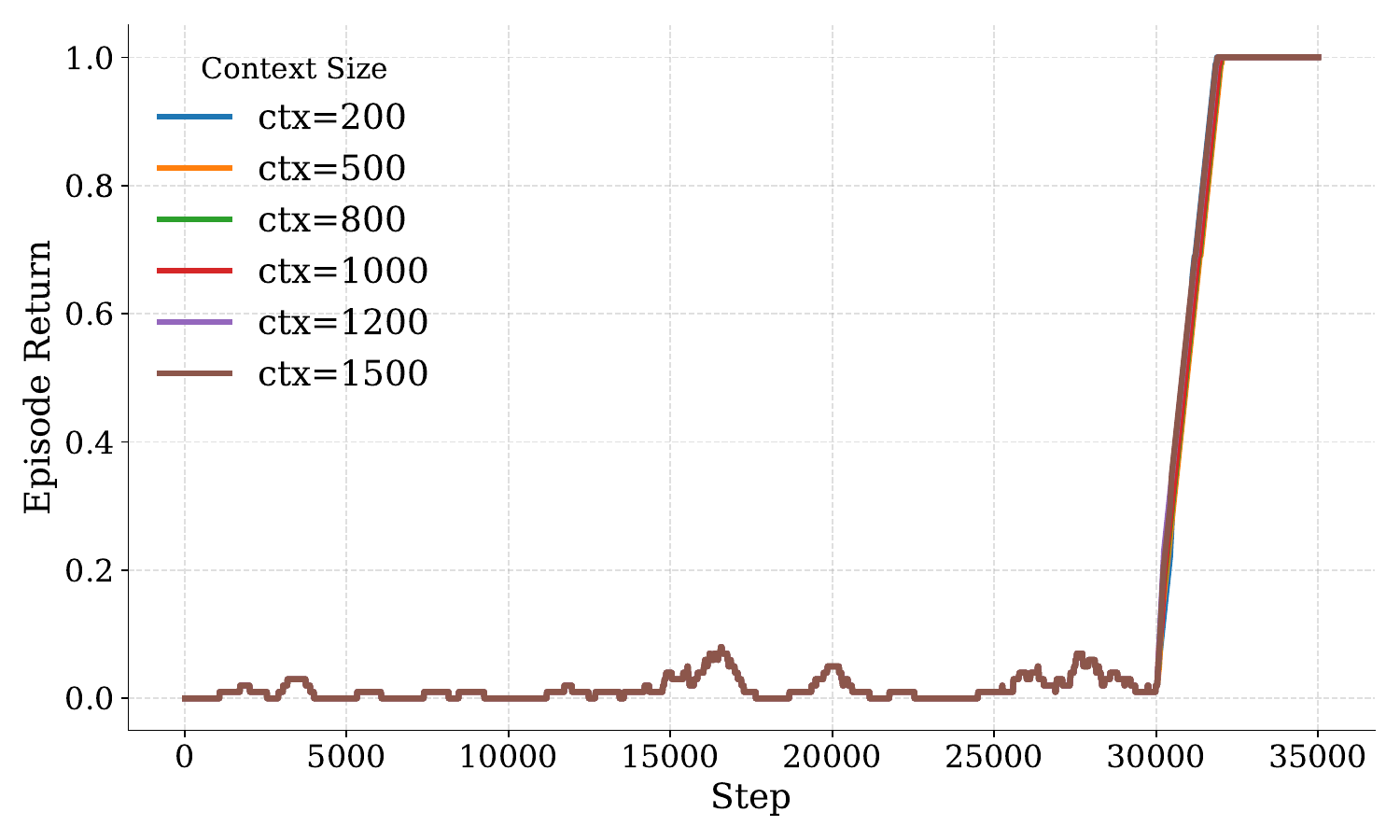}
    \caption{Impact of context size in Frozen Lake.}
    \label{fig:context_length_frozenlake}
    \vspace{-0.1in}
\end{figure}

\paragraph{Replay Buffer and Context Construction.}
In TabQL, the replay buffer stores state, action, and $Q$ value information rather than raw transition tuples. Specifically, for each visited state, we record the estimated $Q$-values for all available actions. During context construction, we sample a fixed number of the most recent states from the replay buffer according to the specified sampling strategy (see Table~\ref{tab:tabql_hyperparams}). For each sampled state, its corresponding set of action-specific $Q$-values is included in the in-context dataset.
Different environments may be associated with distinct action sets and $Q$-value distributions. TabQL treats each state-action pair as an independent query instance during in-context learning, which helps the tabular model to rapidly adapt its value estimation behavior.

\paragraph{Sampling Strategy.}
We also evaluated several sampling strategies for constructing the TabQL context, such as uniform random sampling, distance-based sampling that prioritizes states close to the current prediction state, and uniform sampling across state distances. After various testing experiments, sampling the most recent states consistently yielded the best overall performance across environments. We hypothesize that recent-state sampling better captures the current policy's latest value structure, providing more relevant in-context information for value estimation.

\begin{table}[t]
\centering
\scriptsize
\caption{TabQL hyperparameters for gridworld environments. Unless otherwise specified, all remaining hyperparameters follow standard DQN settings.}
\label{tab:tabql_hyperparams}
\setlength{\tabcolsep}{4pt}
\begin{tabular}{lccccccc}
\toprule
\textbf{Hyperparameter} & \textbf{Cliff} & \textbf{Frozen} & \textbf{Taxi} & \textbf{CartPole} & \textbf{Acrobot} & \textbf{LL} & \textbf{MC} \\
\midrule
Switch Step ($T_0$)              & 20k   & 30k   & 25k   & 20k    & 40k    & 120k   & 80k \\
Context Size ($K$)               & 1k    & 1.5k  & 1k    & 1k     & 2.7k   & 2k     & 2.7k \\
Refit Trigger                    & 1 ep  & 1 ep  & 1 ep  & adapt. & adapt. & adapt. & adapt. \\
Pre-switch Quantile ($q$)        & --    & --    & --    & 0.5    & 0.75   & 0.75   & 0.75 \\
Filter Tolerance ($\tau$)        & --    & --    & --    & 0.1    & 0.1    & 0.1    & 0.1 \\
Failure Floor                    & --    & --    & --    & --     & --     & --     & $-200$ \\
Sampling Strategy                & rec.  & rec.  & rec.  & rec.   & rec.   & rec.   & rec. \\
\bottomrule
\end{tabular}
\vspace{-0.35cm}
\end{table}
    
\paragraph{Effect of Switching Point on FrozenLake.}
Figure~\ref{fig:frozenlake_switching_plot} presents the effect of different switching points on the FrozenLake environment. Across a wide range of switching points (from 5k to 35k), TabQL shows consistently strong performance, which indicates that the precise choice of switching point has a limited impact once sufficient experience has been collected. However, when the switching point is set to an extremely small value (e.g., 50 or 100 steps), as shown in Figure~\ref{fig:frozenlake_switching_plot_low}, TabQL struggles to recover and fails to achieve reliable success.  This behavior suggests that switching to ICL too early can lead to insufficient or misleading context. These observations further support the claim in the main results section that a minimal amount of DQN training is necessary before activating TabQL.

\paragraph{Effect of Context Size on FrozenLake.}
We also investigate TabQL's sensitivity to context size in the FrozenLake environment, while fixing the switching point at 30k steps. As shown in Figure~\ref{fig:context_length_frozenlake}, we evaluate multiple context lengths ranging from 200 to 1500. Across this range, TabQL shows consistent performance with only minor differences between context sizes.

\subsection{Anonymous Code Release}

To support reproducibility, we provide an anonymous GitHub repository containing the TabQL implementation. The repository will be made publicly available upon acceptance.   Our code is available \href{https://anonymous.4open.science/r/TabQL-C045/}{here}.

\subsection{Training Configuration}
All experiments were conducted on a workstation equipped with an Intel(R) Core(TM) i7-14700 CPU and an NVIDIA RTX 4090 GPU.


\subsection{Tabular Foundation Models for Q-Value Estimation}
In this work, we integrate two TFMs, TabPFN~\cite{hollmann2022tabpfn} and TabDPT~\cite{ma2024tabdpt}, into our framework. Both models are well-suited for learning from tabular and sequential data. Those models have demonstrated strong performance in regression settings. We use this inspiration for the classical Q-table formulation. We then adopt these TFMs to model the state–action value function by directly predicting Q-values from tabular representations.

\newpage

\end{document}